\documentclass[letterpaper]{article} 
\usepackage{natbib}
\usepackage{atbegshi}
\renewcommand\cite{\citep}

\usepackage{fancyhdr}
\usepackage{times}  
\usepackage{helvet} 
\usepackage{courier}  
\usepackage[hyphens]{url} 
\usepackage{graphicx} 
\usepackage{natbib}  
\usepackage{caption} 
\usepackage{amsmath}
\usepackage{amssymb}
\usepackage{mathtools, nccmath}
\usepackage{subcaption}
\usepackage{arydshln}
\usepackage[margin=1.5in]{geometry}
\usepackage{paralist}

\usepackage{xcolor}
\usepackage[ruled,vlined,linesnumbered]{algorithm2e}
\usepackage{amsthm}
\usepackage{thmtools,thm-restate}

\title{Evaluating and Mitigating Bias in Image Classifiers: A Causal Perspective Using Counterfactuals}
\author{
   \textbf{Saloni Dash$^1$},
   \textbf{Vineeth N Balasubramanian$^2$},
   \textbf{Amit Sharma$^1$} \\
  \small \texttt{ t-sadash@microsoft.com},
  \texttt{vineethnb@iith.ac.in} \\
  \small \texttt{amshar@microsoft.com} \\
  
\small 1 - Microsoft Research, Bangalore, Karnataka, India \\
\small 2 - Indian Institute of Technology Hyderabad, Telangana, India
}

\newif{\ifhidecomments}
\hidecommentsfalse
\ifhidecomments
    \newcommand{\saloni}[1]{}
    \newcommand{\amit}[1]{}
\else
    \newcommand{\saloni}[1]{\textcolor{blue}{[#1 ---\textsc{saloni}]}}
    \newcommand{\amit}[1]{\textcolor{red}{[#1 ---\textsc{amit}]}}
\fi

\newcommand{\doo}{\operatorname{do}}
\newcommand{\E}{\mathbb{E}}
\newcommand{\vecx}{\mathbf{x}}
\newcommand{\veca}{\mathbf{a}}
\newcommand{\vecX}{\mathbf{X}}
\newcommand{\vecA}{\mathbf{A}}
\newcommand{\vecz}{\mathbf{z}}
\newcommand{\vecZ}{\mathbf{Z}}

\newcommand{\ourmethod}{\emph{ImageCFGen}}
\newcommand{\norm}[1]{\left\lVert#1\right\rVert}
\fancypagestyle{preprint}{\fancyhf{}\fancyfoot[L]{Accepted for Publication at WACV 2022.}}
\begin{document}

\date{}

\maketitle

\thispagestyle{preprint}

\begin{abstract}
Counterfactual examples for an input---perturbations that change specific features but not others---have been shown to be useful for evaluating bias of machine learning models, e.g., against specific demographic groups. However, generating counterfactual examples for images is non-trivial due to the underlying causal structure on the various features of an image. To be meaningful, generated perturbations need to satisfy constraints implied by the causal model. We present a method for generating counterfactuals by incorporating a structural causal model (SCM) in an improved variant of Adversarially Learned Inference (ALI), that  generates counterfactuals in accordance with the causal relationships between attributes of an image. Based on the generated counterfactuals, we show how to explain a pre-trained machine learning classifier, evaluate its bias, and mitigate the bias using a counterfactual regularizer. On the Morpho-MNIST dataset, our method generates counterfactuals comparable in quality to prior work on SCM-based counterfactuals (DeepSCM), while on the more complex CelebA dataset our method outperforms DeepSCM in generating high-quality valid counterfactuals. Moreover, generated counterfactuals are indistinguishable from reconstructed images in a human evaluation experiment and we subsequently use them to evaluate the fairness of a standard classifier trained on CelebA data. We show that the classifier is biased w.r.t. skin and hair color, and how counterfactual regularization can remove those biases.
\end{abstract}

\section{Introduction}
A growing number of studies have uncovered biases in image classifiers, particularly against marginalized demographic groups \cite{buolamwini2018gender, hendricks2018women, zhao2017men, bolukbasi2016man}. To avoid such biases, it is important to explain a classifier's predictions and evaluate its fairness.
Given any pre-trained machine learning (ML) classifier, counterfactual reasoning is an important way to explain the classifier's decisions and to evaluate its fairness. Counterfactual reasoning involves simulating an alternative input with some specific changes compared to the original input.  For example, to evaluate fairness of a classifier with respect to a sensitive attribute like \textit{skin color}, we can ask how the classifier's output will change if a face that was originally \textit{dark-skinned} is made \textit{light-skinned} while keeping everything else constant. Since the only change to the input is the sensitive attribute, counterfactual reasoning is considered more robust than comparing available faces with different skin colors (association) or comparing simulated inputs with \textit{light skin} or \textit{dark skin} (intervention) since these comparisons may include changes in addition to \textit{skin color}.

However, generating counterfactual (CF) examples for images is non-trivial.
For instance, consider the simple task of changing a person's hair color in an image of their face. While Generative Adversarial Networks (GANs) can generate new realistic faces with any hair color \cite{kocaoglu2018causalgan}, they are unable to generate the precise changes needed for a CF, i.e. changing hair color without changing hair style or other features of the face. Other explanation techniques based on perturbations such as occluding pixels~\cite{zeiler2014visualizing} also do not support counterfactual reasoning based on high-level concepts. 

There have been recent efforts on using GANs to generate counterfactuals using an added inference step (encoder).  Given a pre-trained GAN model, \citet{denton2019detecting} trained an encoder to match the input of a generated image. However, the latents thus encoded do not directly correspond to the given attributes of an image, and it is difficult to change a specific known attribute to generate a counterfactual. To change an attribute, Joo and K{\"a}rkk{\"a}inen  \cite{joo2020gender} used the FaderNetwork architecture that inputs attributes of an image separately to the generator. However, their method does not account for causal relationships between attributes. Besides, while both these works use generated images to evaluate biases in a classifier, they do not provide any method to mitigate the found biases. 
We present a method for generating counterfactuals that is based on their formal causal definition, and present a novel counterfactual-based regularizer to mitigate biases in a given classifier. 

Formally, a valid counterfactual example for an image is defined with respect to a Structural Causal Model (SCM) over its attributes.
An SCM encodes the domain knowledge about how attributes affect each other in the form of a graph with attributes as the nodes and accompanying functional equations connecting the nodes. Generating a counterfactual, therefore, requires modeling both the underlying SCM for the attributes as well as the generative process that uses the attributes to model the resulting image. 
\newline \indent
In this paper, we present \ourmethod, a method that combines knowledge from a causal graph and  uses an inference mechanism in a GAN-like framework to generate counterfactual images.
We first evaluate \ourmethod\ on the Morpho-MNIST dataset to show that it generates counterfactual images comparable to a prior SCM-based CF generation method (DeepSCM) \cite{pawlowski2020deep}. Moreover, we show that our method is capable of generating high-quality valid counterfactuals for complex datasets like CelebA in comparison to DeepSCM.

Specifically, on the CelebA dataset, \ourmethod\ can generate CFs for facial attributes like \textit{Black Hair}, \textit{Pale Skin} etc. Based on these counterfactual images, we show that an image classifier for predicting attractiveness on the CelebA dataset exhibits bias with respect to \textit{Pale Skin}, but not with respect to attributes like \textit{Wearing Necklace}. We hence propose and demonstrate a bias mitigation algorithm which uses the counterfactuals to remove the classifier's bias with respect to sensitive attributes like \textit{Pale Skin}.
In summary, our contributions include: 
\begin{compactitem}
\item \ourmethod, a method that uses inference in a GAN-like framework to generate counterfactuals based on attributes learned from a known causal graph. 
\item  Theoretical justification that under certain assumptions, CFs generated by \ourmethod\ satisfy the definition of a counterfactual as in Pearl \cite{pearl2009book}.
\item Detailed experiments on Morpho-MNIST and CelebA datasets that demonstrate the validity of CFs generated by \ourmethod\ in comparison to DeepSCM \cite{pawlowski2020deep}.

\item Evaluating fairness of an image classifier and explaining its decisions using counterfactual reasoning. 
\item A regularization technique using CFs to mitigate bias w.r.t. sensitive attributes in any image classifier.
\end{compactitem}

\section{Related Work}
Our work bridges the gap between generating counterfactuals and evaluating fairness of image classifiers. 

\subsection{Counterfactual Generation}
Given the objective to generate a data point $X$ (e.g., an image) based on a set of attributes $A$, Pearl's ladder of causation \cite{pearl2019seven} describes three types of distributions that can be used: \textit{association} $P(X|A=a)$, \textit{intervention} $P(X|\doo(A=a))$ and \textit{counterfactual} $P(X_{A=a} | A=a', X=x')$. 
In the associational distribution $P(X|A=a)$, $X$ is conditioned on a specific attribute value $a$ in the observed data. For e.g., conditional GAN-based methods \cite{mirza2014conditional} implement association between attributes and the image. In the interventional distribution $P(X|\doo(A=a)$, $A$ is changed to the value $a$ irrespective of its observed associations with other attributes. Methods like CausalGAN \cite{kocaoglu2018causalgan} learn an interventional distribution of images after changing specific attributes, and then generate new images that are outside the observed distribution (e.g., women with moustaches on the CelebA faces dataset \cite{kocaoglu2018causalgan}).
\newline \indent
The counterfactual distribution $P(X_{A=a}| A=a', X=x')$ denotes the distribution of images related to a specific image $x'$ and attribute $a'$, if the attribute of the same image is changed to a different value $a$. In general, generating counterfactuals is a harder problem than generating interventional data since we need to ensure that everything except the changed attribute remains the same between an original image and its counterfactual. \citet{pawlowski2020deep} recently proposed a method for generating image counterfactuals using a conditional Variational Autoencoder (VAE) architecture. While VAEs allow control over latent variables, GANs have been more successful over recent years in generating high-quality images. Thus, we posit that a GAN-based method is more ideally suited to the task of CF generation in complex image datasets, especially when the generated images need to be realistic to be used for downstream applications such as fairness evaluation. We test this hypothesis in Section~\ref{sec:results}. 
\newline \indent
Independent of our goal, there is a second interpretation of a ``counterfactual'' example w.r.t. a ML classifier~\cite{wachter2017counterfactual}, referring to the smallest change in an input that changes the prediction of the ML classifier.  \cite{liu2019generative} use a standard GAN with a distance-based loss to generate CF images close to the original image. However, the generated CFs do not consider the underlying causal structure -- terming such images as CFs can be arguable from a causal perspective. Besides, these CFs are not perceptually interpretable -- ideally, a counterfactual image should be able to change only the desired attribute while keeping the other attributes constant, which we focus on in this work.

\subsection{Fairness of Image Classifiers}
Due to growing concerns on bias against specific demographics in image classifiers \cite{buolamwini2018gender},  methods have been proposed to inspect what features are considered important by a classifier~\cite{sundararajan2017axiomatic}, constrain classifiers to give importance to the correct or unbiased features~\cite{ross2017right}, or enhance fairness by generating realistic images from under-represented  groups~\cite{mcduff2019bias}. 
Explanations, to study the fairness of a trained model, can also be constructed by perturbing  parts of an image that change the classifier's output, e.g., by  occluding an image region~\cite{zeiler2014visualizing} or by changing the smallest parts of an image that change the classifier's output~\cite{fong2017interpretable,luss2019generating,goyal2019counterfactual}. Such perturbation-based methods, however, are not designed to capture high-level concepts and do not enable study of fairness of classifiers w.r.t. human-understandable concepts (e.g. gender or race). 
\newline \indent 
Existing work closest to our efforts include two adversarially-trained generative models to generate counterfactuals for a given image. 
\cite{denton2019detecting} changed attributes for a given image by linear interpolations of latent variables using a standard progressive GAN~\cite{karras2018progressive}. Similarly,  
\cite{joo2020gender} used a Fader network architecture~\cite{lample2017fader} to change attributes. However, both these works ignore the causal structure associated with attributes of an image. In analyzing bias against an attribute, it is important to model the downstream changes \emph{caused} by changing that attribute~\cite{kusner2017counterfactual}. For instance, for a chest MRI classifier, age of a person may affect the relative size of their organs~\cite{pawlowski2020deep}; it will not be realistic to analyze the effect of age on the classifier's prediction without learning the causal relationship from age to organ size. 
Hence, in this work, we present a different architecture that can model causal relationships between attributes and provide valid counterfactuals w.r.t. an assumed structural causal model. In addition,  using these counterfactuals, 
we present a simple regularization technique  that can be used to decrease bias in any given classifier.
\section{SCM-Based Counterfactuals}
Let  $\vecX = \vecx \in \mathcal{X}$ denote the image we want to generate the counterfactual for, and let $\vecA = \veca = \{a_i\}_{i=1}^{n} \in \mathcal{A}$ be its corresponding attributes. In the case of human faces, attributes can be binary variables ($\in \{0,1\}$) like \textit{Smiling}, \textit{Brown Hair}; or in the case of MNIST digits, continuous attributes ($\in \mathbb{R} $) like \textit{thickness}, \textit{intensity}, etc.
A continuous attribute is scaled so that it lies in the range of $[0,1]$. We have a training set $\mathcal{D}$ containing ($\vecx$, $\veca$) (image, attribute) tuples. Given an image ($\vecx$, $\veca$), the goal is to  generate a counterfactual image with the attributes changed to $\veca_c$.

\subsection{SCM over Attributes}\label{subsec_scm_over_attribs}
We assume an SCM for the true data-generating process that defines the relationship among the attributes $\veca$, and from the attribute to the image $\vecx$. For instance, with two binary attributes (\textit{Young}, \textit{Gray Hair}) for an image $\vecx$ of a face,  the true causal graph can be assumed to be $\textit{Young} \to \textit{Gray Hair} \to \vecx$. We separate out the graph into two parts: relationships amongst the attributes ($\mathcal{M}_a$), and relationships from the attributes to the image ($\mathcal{M}_x$). We call $\mathcal{M}_a$ as the \textit{Attribute-SCM} and model $\mathcal{M}_x$ as a generative model, given the attributes. 

The Attribute-SCM ($\mathcal{M}_a$) consists of a causal graph structure and associated structural assignments. We assume that the causal graph structure is known. Given the graph structure, Attribute-SCM learns the structural assignments between attributes. E.g., given \textit{Young} $\rightarrow$ \textit{Gray Hair}, it learns the function $g$ such that $\textit{Gray Hair} = g(\textit{Young}, \epsilon)$, where $\epsilon$ denotes independent random noise.
We use the well-known Maximum Likelihood Estimation procedure in Bayesian Networks \cite{ding2010probabilistic} to estimate these functions for the attributes, but other methods such as  Normalizing Flows \cite{rezende2015variational,pawlowski2020deep} can also be used.

 Note that counterfactual estimation requires knowledge of both the \textit{true} causal graph and the \textit{true} structural equations; two SCMs may entail exactly the same observational and interventional distributions, but can differ in their counterfactual values~\cite{karimi2020algorithmic}. In most applications, however,  it is impractical to know the true structural equations for each edge of a causal graph. Therefore, here we make a simplifying empirical assumption: while we assume a known causal graph for attributes, we estimate the structural equations from observed data. That is, we assume that the data is generated from a subset of all possible SCMs (e.g., linear SCMs with Gaussian noise) such that  the structural equations can be uniquely estimated from data. Details on Attribute-SCM are in Appendix A.

\subsection{Image Generation from Attribute-SCM}
For modeling the second part of the SCM  from attributes to the image ($\mathcal{M}_x$), we build a generative model that contains an encoder-generator architecture. 
Given an Attribute-SCM (either provided by the user or learned partially, as in Sec \ref{subsec_scm_over_attribs}), the proposed method \ourmethod\ has  three steps to generate a counterfactual for an image $(\vecx,\veca)$ such that the attributes are changed to $\veca'$ (architecture in Figure~\ref{cfgen}). 
\begin{compactitem}
    \item An encoder $E: (\mathcal{X}, \mathcal{A}) \to \vecZ$ infers the latent vector $\vecz$ from $\vecx$ and $\veca$, i.e. $\vecz=E(\vecx, \veca)$ where $\vecZ = \vecz \in \mathbb{R}^m$. 
    \item The Attribute-SCM  intervenes on the desired subset of attributes that are changed from $\veca$ to $\veca'$, resulting in output $\veca_c$. Specifically, let $\veca_k \subseteq \veca$ be the subset of attributes that changed between the inputs $\veca$ and $\veca'$. For every $a_i\in \veca_k$, set its value to $a'_i$, then change the value of its descendants in the SCM graph by plugging in the updated values in the structural equations (see Lines 5-6 in Algorithm \ref{algorithm1}, Appendix A).
    \item Generator $G: (\vecZ, \mathcal{A}) \to \mathcal{X}$ takes as input $(\vecz, \veca_c)$ and generates a counterfactual $\vecx_c$, where $\vecz \in \vecZ \subseteq \mathbb{R}^m$. 
\end{compactitem}

\begin{figure}[h]
\centering
\includegraphics[width = 0.9\columnwidth]{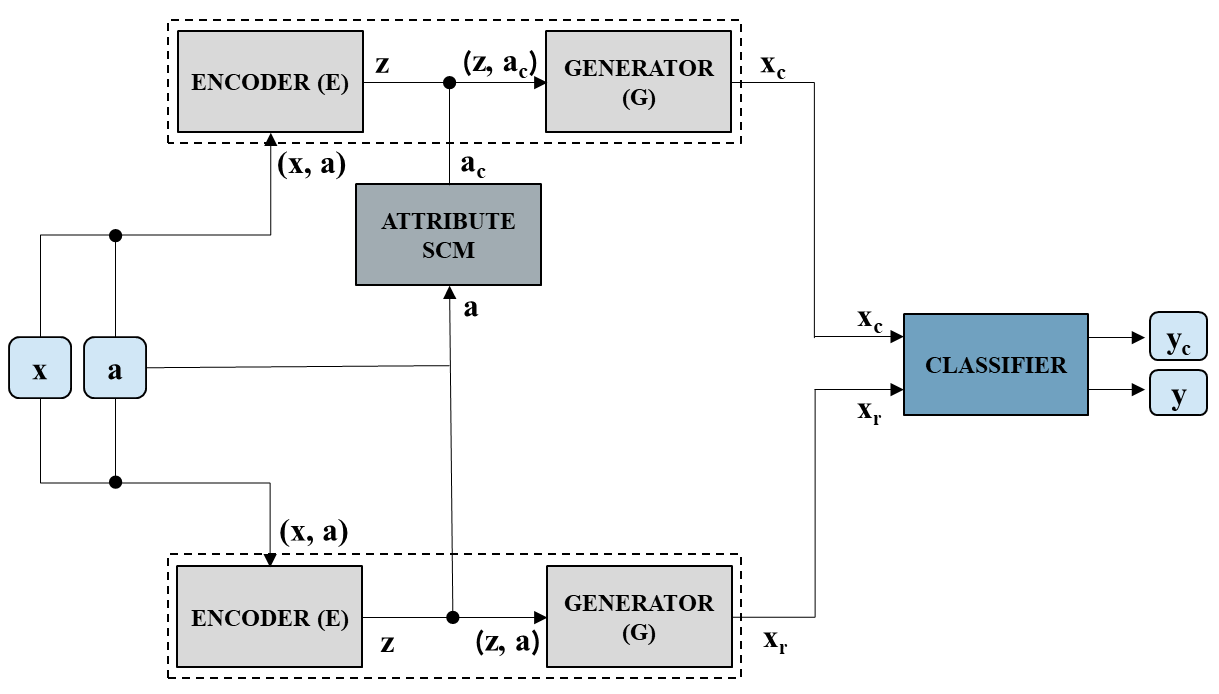}
\caption{\textbf{Counterfactual Generation using \ourmethod}. The top half of the figure shows the CF generation procedure, and the bottom half of the figure shows the reconstruction procedure.
Finally, the reconstructed image $x_r$ and the counterfactual image $x_c$ are used for a downstream task like fairness evaluation of a classifier. }

\label{cfgen}
\end{figure}

The above method for CF generation can be written as: 
\begin{equation}\label{counterfactual}
\vecx_c = G(E(\vecx,\veca), \veca_c)
\end{equation}

The complete algorithm is shown in Algorithm~\ref{algorithm1} in Appendix A. For our experiments, we use a novel improved variant of Adversarially Learned Inference \cite{dumoulin2016adversarially} to train the encoder and generator. However, \ourmethod\ can be extended to any encoder-generator (decoder) architecture.

\subsection{Correspondence to Counterfactual Theory}
The above architecture maps directly to the three steps for generating a valid counterfactual, for $(\vecx,\veca)$ as in~\cite{pearl2009book}:
\begin{compactitem}
    \item \textbf{Abduction}: Infer latent $\vecz$ given the input $(\vecx, \veca)$ using the encoder.
    \item \textbf{Action}: Let $\veca_k \subseteq \veca$ be the set of $k$ attributes that one wants to intervene on. Set attribute $a_i \rightarrow a'_i $  $\forall a_i \in \veca_k$, where $\veca'_k = \{a'_i\}_{i=1}^k$.
    \item \textbf{Prediction}: Modify all the descendants of $\veca_k$ according to the SCM equations learned by Attribute-SCM. This outputs $\veca_c$, the intervened attributes. Use $\vecz$ from the encoder and $\veca_c$ from the Attribute-SCM and input it to the generator to obtain the counterfactual $x_c$.
\end{compactitem}

The proof that Equation \ref{counterfactual}  corresponds to generating a valid counterfactual is in Appendix B.

\subsection{Implementing the Encoder and Generator}
Many studies have reported generating high-quality, realistic images using GANs \cite{wang2018high, karras2018progressive, karras2019style}. However, vanilla GANs lack an inference mechanism where the input $\vecx$ can be mapped to its latent space representation $\vecz$. We hence use Adversarially Learned Inference (ALI) \cite{dumoulin2016adversarially}, which integrates the inference mechanism of variational methods like VAEs in a GAN-like framework, thereby leveraging the generative capacity of GANs as well as providing an inference mechanism. We generate the images using a conditional variant of ALI where the model is conditioned on the attributes $\veca$ while generating an image. 
\newline \indent 
An ALI-based method, however, has two limitations: (1) generation capabilities are limited when compared to state-of-the-art \cite{karras2019style}; and (2) reconstructions are not always faithful reproductions of the original image~\cite{dumoulin2016adversarially}. Image reconstructions are important in the counterfactual generation process because they indicate how good the inferred latent space variable $z$ is, which is used in the abduction step of generating counterfactuals.  We address both these issues by using a style-based generator for better generation and a cyclic cost minimization algorithm for improved reconstructions. We refer to our ALI model as Cyclic Style ALI (CSALI) and describe each of these components below.
\newline \noindent  \textbf{Adversarially Learned Inference.}
ALI uses an encoder $E$ and a generator $G$ in an adversarial framework, where the encoder learns to approximate the latent space distribution from the input $\vecx$ and attributes $\veca$, and the generator learns to generate realistic images from the latent space distribution and attributes $\veca$. The discriminator $D$ is optimized to differentiate between pairs of tuples containing $\{$the real image, the corresponding approximated latent space variable, attributes$\}$ from joint samples of $\{$the generated image, the input latent variable, attributes$\}$. The Encoder and Generator are optimized to fool the discriminator. Unlike \cite{dumoulin2016adversarially} which uses an embedding network, we directly pass the attributes to the Generator, Encoder and Discriminator since we found that it helped in conditioning the model on the attributes better. The conditional ALI hence optimizes:
\begin{equation}
\begin{medsize}
\begin{aligned}
\min_{G,E}\max_{D}V(D,G,E) = \mathbb{E}_{q(\vecx)p(\veca)}[\log(D(\vecx, E(\vecx,\veca), \veca))] 
\\
+ \mathbb{E}_{p(\vecz)p(\veca)}[\log(1 - D(G(\vecz, \veca), \vecz, \veca))]
\end{aligned}
\end{medsize}
\end{equation}
\noindent where $G$ is the generator, $E$ is the encoder and $D$ is the discriminator. $q(\vecx)$ is the distribution for the images, $p(z) \sim \mathcal{N}(0,1)$ and $p(a)$ is the distribution of the attributes. 
Image reconstructions are defined as:
\begin{equation}\label{reconstructed}
 \vecx_r = G(E(\vecx,\veca), \veca)
 \end{equation}
 \newline \noindent 
\textbf{Style-Based Generator.}
Style-based generator architectures (StyleGANs) implicitly learn separations of high-level attributes of images like hair color, pose, etc \cite{karras2019style,karras2020analyzing}, and generate images that are indistinguishable from real images \cite{zhou2019hype}. To improve generation, we replace the ALI generator with the style-based generator architecture in \cite{karras2019style}. Details of the architecture are provided in the Appendix D.
\newline \noindent
\textbf{Cyclic Cost Minimization.}
To improve image reconstructions (which in turn indicate the goodness of the learned latents $\vecz$), we employ the cyclic cost minimization algorithm in \cite{dogan2020semi} after training the style-based ALI model. The generator is fixed, and the encoder is fine-tuned to minimize a reconstruction loss computed using: (i) error in the image space $\mathcal{L}_{\vecx} = \mathbb{E}_{\vecx \sim q(\vecx)}\norm{\vecx - G(E(\vecx,\veca))}_2$; and (ii) error in the latent space $\mathcal{L}_{\vecz} = \mathbb{E}_{\vecz \sim p(\vecz)}\norm{\vecz - E(\vecx,\veca)}$, where $G(E(\vecx,\veca))$ is the reconstructed image $\vecx_r$ according to Eqn \ref{reconstructed} and $E(\vecx,\veca)$ is the encoder's output $\vecz_r$, which is expected to capture the image's latent space distribution. We fine-tune the encoder using the above reconstruction loss \emph{post-hoc} after obtaining a good generator in order to explicitly improve image reconstructions. 

\section{Applications of Generated CFs}
We now show how the counterfactuals generated using \ourmethod can be used to evaluate fairness of, as well as explain a given image classifier. We will also present a method to mitigate any fairness biases in the classifier. Suppose we are given a pre-trained image classifier $\hat{f}:\mathcal{X} \to \mathcal{Y}$, such that $\hat{f}(\vecx) = \hat{y}$, where $\vecx \in \mathcal{X}$ refers to the images and $ \hat{y} \in \mathcal{Y}$ refers to the classifier's discrete outcome. Let $\veca \in \mathcal{A}$ be the corresponding image attributes, and let $\veca_S \in \mathcal{A_S} \subseteq \mathcal{A}$ be the set of sensitive attributes we want to evaluate the classifier on.

\subsection{Evaluating Fairness of a Classifier} 
We can use the generated CFs to estimate biases in a given classifier that predicts some discrete outcome $\hat{y} = y$ (like \textit{Attractive}). In an ideal scenario, the latent variable $\vecz$ for the real image and its reconstructed image would match exactly. However, experiments using ALI demonstrate that the reconstructed images are not perfect reproductions of the real image \cite{dumoulin2016adversarially, donahue2019large, dandi2020generalized}. Therefore, for objective comparison, we compare classification labels for reconstructed images ($\vecx_r$ from Eqn \ref{reconstructed}) and counterfactual images ($\vecx_c$ from Eqn \ref{counterfactual}), since the reconstructed images share the exact same latent $\vecz$ as the CF image (and hence the CF will be valid). We hence refer to the reconstructed images (which share the latent $\vecz$ with the CF) as \textit{base} images for the rest of the paper.

We characterize a classifier as \emph{biased} w.r.t. an attribute if: (a) it changes its classification label for the CF image (obtained by changing that attribute); and (b) if it changes the label to one class from another class more often than vice versa (for CFs across test images obtained by changing the considered attribute). (To illustrate the second condition, if setting hair color as blonde makes test images consistently be classified as attractive more often than otherwise, this indicates bias.)  We capture these intuitions as a formula for the degree of bias in a binary classifier w.r.t. a considered attribute:
\begin{equation}\label{bias}
\begin{aligned}
    \text{bias} = p(y_r \neq y_c)(p(y_r = 0, y_c = 1 |  y_r \neq y_c) 
    \\
    - p(y_r = 1, y_c = 0 | y_r \neq y_c))
    \end{aligned}
\end{equation}
where $y_r$ is the classification label for the reconstructed image, and $y_c$ is the classification label for the CF image. Using Bayes Theorem, Eqn \ref{bias} reduces to: 
\begin{equation}\label{bias-simple}
 \text{bias} = p(y_r = 0, y_c = 1) - p(y_r = 1, y_c = 0)
 \end{equation}
The bias defined above ranges from -1 to 1. It is 0 in the ideal scenario when the probability of CF label changing from 0 to 1 and vice-versa is the same (=0.5). The bias is 1 in the extreme case when the CF label always changes to 1, indicating that the classifier is biased \textit{towards} the counterfactual change. Similarly if the CF label always changes to 0, the  bias will be -1, indicating that the classifier is biased \textit{against} the counterfactual change. 
In Appendix C, we show that a classifier with zero bias in the above metric is fair w.r.t. the formal definition of counterfactual fairness~\cite{kusner2017counterfactual}.

\subsection{Explaining a Classifier}
We can also use the CFs from \ourmethod\ to generate explanations for a classifier. For any input $\vecx$, a local \textit{counterfactual importance score} for an attribute $\veca_i$ states how the classifier's prediction changes upon changing the value of $\veca_i$. Assuming $\veca_i$ can take binary values $a'=1$ and $a=0$, the local \textit{CF importance score} of $\veca_i$ is given by:
\begin{equation}
\begin{split}
    \E_Y[Y_{\veca_i \leftarrow a'}| \vecx, \veca] - \E_Y[Y_{\veca_i \leftarrow a}| \vecx, \veca] 
    \\
     =y_{\veca_i \leftarrow a'}|\vecx, \veca - y_{\veca_i \leftarrow a}|\vecx, \veca
    \end{split}
    \end{equation}
    
where $Y$ is the random variable for the classifier's output, $y$ is a value for the classifier's output, and the above equality is for a deterministic classifier. For a given $(\vecx, \veca)$, the score for each attribute (feature) can be ranked to understand the relative importance of features. To find global feature importance, we average the above score over all inputs.

\subsection{Bias Mitigation for a Classifier}
Finally, in addition to evaluating a classifier for bias, CFs generated using \ourmethod can be used to remove bias from a classifier w.r.t. a sensitive attribute. Here we propose a \textit{counterfactual} regularizer to ensure that an image and its counterfactual over the sensitive attribute obtain the same prediction from the image classifier.  For an image $\vecx$, let $logits(\vecx)$ be the output of the classifier $\hat{f}$  before the sigmoid activation layer. To enforce fairness, we can finetune the classifier by adding a regularizer that the logits of the image and its counterfactual should be the same, i.e.

\begin{equation}
    \text{BCE}(y_{true}, \hat{f}(\vecx)) + \lambda \text{MSE}(\text{logits}(\vecx_r), \text{logits}(\vecx_c))
    \label{eq:bias-mitigation}
\end{equation}
where BCE is the binary cross-entropy loss, $y_{true}$ is the ground truth label for the real image $\vecx$, $\lambda$ is a regularizing hyperparameter, MSE is the mean-squared error loss, and $\vecx_r$ and $\vecx_c$ are defined in Eqns \ref{reconstructed} and \ref{counterfactual} respectively.

\section{Experiments and Results}\label{sec:results}
Considering the limited availability of datasets with known causal graphs, we study \ourmethod\ on the Morpho-MNIST dataset (a simple dataset for validating our approach), and on the CelebA dataset (which provides an important context for studying bias and fairness in image classifiers). Specifically, we study the following: 
\begin{compactitem}
    \item \textbf{Validity of \ourmethod\ CFs.} We use the Morpho-MNIST dataset which adds causal structure on the MNIST images, to compare counterfactuals from \ourmethod\ to the Deep-SCM method~\cite{pawlowski2020deep}. We show that CFs from \ourmethod\ are comparable to those from DeepSCM, thus validating our approach.
    \item \textbf{Quality of \ourmethod\ CFs.} On the more complex CelebA dataset, we evaluate the quality of \ourmethod\ CFs by quantifying the generation and reconstruction, using established benchmark metrics. We find that using the proposed CSALI architecture offers significant advantages over the standard ALI model. We also contrast the quality and validity of the generated CFs with those of DeepSCM, and find that \ourmethod\ outperforms DeepSCM.
    \item \textbf{Fairness Evalution and Explanation of a ML Classifier Using  \ourmethod\ CFs.} We show using \ourmethod\ that a standard pre-trained classifier on CelebA that predicts whether a face is attractive or not, has bias w.r.t. attribute \textit{Pale Skin} across all three hair colors (\textit{Black Hair}, \textit{Blond Hair}, \textit{Brown Hair}). We also explain the classifier's predictions using CFs. 
    \item \textbf{Bias Mitigation of a ML Classifier Using \ourmethod\ CFs.} Finally, we show how our proposed method can be used to decrease detected bias in the classifier for the attributes mentioned above.   
\end{compactitem}  
\noindent \textbf{Baselines and Performance Metrics}. We compare to DeepSCM's \cite{pawlowski2020deep} results on Morpho-MNIST and CelebA.
We present both quantitative and qualitative performance of our method for these datasets. While we follow the metrics of \cite{pawlowski2020deep} for Morpho-MNIST, for CelebA we report quantitative scores like Fréchet Inception Distance \cite{heusel2017gans} (FID) and Mean Squared Error (MSE) to compare generation and reconstruction quality with the base ALI method. For measuring quality of generated counterfactuals, we report  human evaluation scores, in addition to qualitative results.  For bias evaluation, we compare \ourmethod\ to affine image transformations like horizontal flip and brightness that are commonly used data augmentation techniques.

\noindent \textbf{Datasets}.
\textit{Morpho-MNIST} \cite{castro2019morpho} is a publicly available dataset based on MNIST \cite{lecun1998gradient} with interpretable attributes like \textit{thickness}, \textit{intensity}, etc. It was extended by \cite{pawlowski2020deep} to introduce morphological transformations with a known causal graph. The attributes are \textit{thickness (t)} and \textit{intensity (i)}, where \textit{thickness} $\to$ \textit{intensity} ($\to$ indicating causal effect). We extend this dataset by introducing an independent morphological attribute---\textit{slant (s)} from the original Morpho-MNIST dataset and digit \textit{label (l)} as an attribute. The causal graph for the dataset is given in Fig \ref{causalgraphmnist}.

\begin{figure}[!htb]
\centering
\begin{subfigure}{0.4\columnwidth}
  \centering
 \includegraphics[width = 0.85\columnwidth]{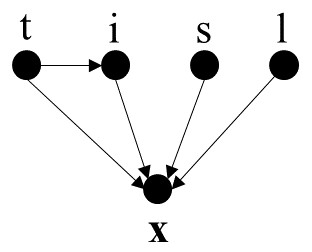}
 \caption{Morpho-MNIST}
 \label{causalgraphmnist}
\end{subfigure}
\begin{subfigure}{0.58\columnwidth}
  \centering
  \includegraphics[width = 0.65\columnwidth]{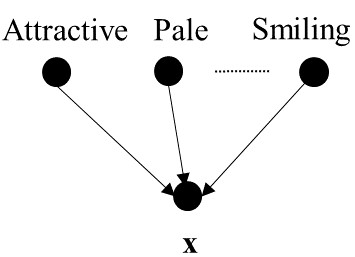}
  \caption{CelebA.}
  \label{celebagraph}
\end{subfigure}
\caption{\textbf{Causal Graphs for Morpho-MNIST and CelebA.} Attributes for Morpho-MNIST are \textit{thickness} t, \textit{intensity} i, \textit{slant} s and \textit{label} l; for CelebA are \textit{Pale}, \textit{Black Hair}, etc. In both graphs, attributes cause the image x.
}
\end{figure}

\textit{CelebA ~\cite{liu2015faceattributes}} is a dataset of 200k celebrity images annotated with 40 attributes like \textit{Black Hair}, \textit{Wearing Hat}, \textit{Smiling} etc. 
We train an image classifier on the dataset that predicts the attribute \textit{Attractive} as done in \cite{sattigeri2019fairness, ehrlich2016facial}. While explaining the classifier's decisions, we generate CFs for all attributes excluding \textit{Male}, \textit{Young} and \textit{Blurry}. For fairness evaluations, we focus on generating CFs for the attributes \textit{Black Hair}, \textit{Blond Hair}, \textit{Brown Hair}, \textit{Pale} and \textit{Bangs}. 
Similar to \cite{denton2019detecting}, we do not generate CFs for \textit{Male} because of inconsistent social perceptions surrounding gender, thereby making it difficult to define a causal graph not influenced by biases. Therefore, all attributes we consider have a clear causal structure (Fig \ref{celebagraph} shows the causal graph). Additionally, our method can also be utilized in the setting where the attributes are connected in a complex causal graph structure, unlike \cite{denton2019detecting, joo2020gender}. We show this by conducting a similar fairness and explanation analysis for a \textit{Attractive} classifier in Appendix O, where \textit{Young} affects other visible attributes like \textit{Gray hair}.

Details of implementation, architecture (including ALI) and training are provided in Appendix D. 

\subsection{Validity of \ourmethod\ CFs on Morpho-MNIST}
We generate CFs using \ourmethod\ on images from the Morpho-MNIST dataset by intervening on all four attributes - \textit{thickness}, \textit{intensity}, \textit{slant} and \textit{label} and observe how the image changes with these attributes. Fig \ref{mnistcnt} demonstrates CFs for a single image with label 0. Along the first column vertically, the label is changed from 0 to $\{1,4,6,9\}$ while the thickness, intensity and slant are kept constant. Then, as we proceed to the right in each row, the attributes of thickness, intensity and slant are changed sequentially but the  label is kept constant. Visually,  the generated counterfactuals change appropriately according to the intervened attributes. For instance, according to the causal graph in Fig \ref{causalgraphmnist}, changing the digit label should not change the digit thickness intensity and slant. That is precisely observed in the first column of Fig \ref{mnistcnt}. Whereas, changing the thickness should also change the intensity of the image which is observed in the third and fourth columns  of Fig \ref{mnistcnt}. Results of latent space interpolations and digit reconstructions are provided in Appendix E and F. We find that the encoder learns meaningful latent space representations, and the reconstructions are faithful reproductions of Morpho-MNIST digits.
\begin{figure}[h]
\centering
\includegraphics[width = 0.85\columnwidth]{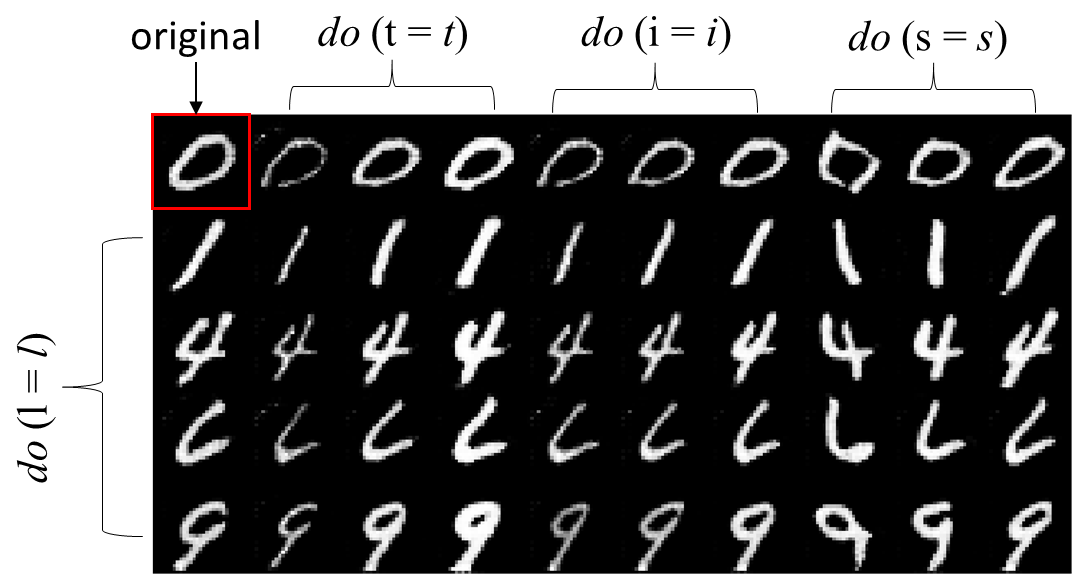}
    \caption{\textbf{Morpho-MNIST Counterfactuals.} Top-left cell shows a real image sampled from the test set. Vertically, rows correspond to interventions on the label, $do$(\textit{l} = 1, 4, 6, 9). Moving horizontally, columns correspond to interventions on thickness: $do$ (\textit{t}=1, 3, 5), intensity: $do$ (\textit{i} = 68, 120, 224), and slant: $do$ (\textit{s} = -0.7, 0, 1) respectively.  
    }
    \label{mnistcnt}
\end{figure}

\begin{figure}[!htb]
    \centering
    \includegraphics[width = 0.8\columnwidth]{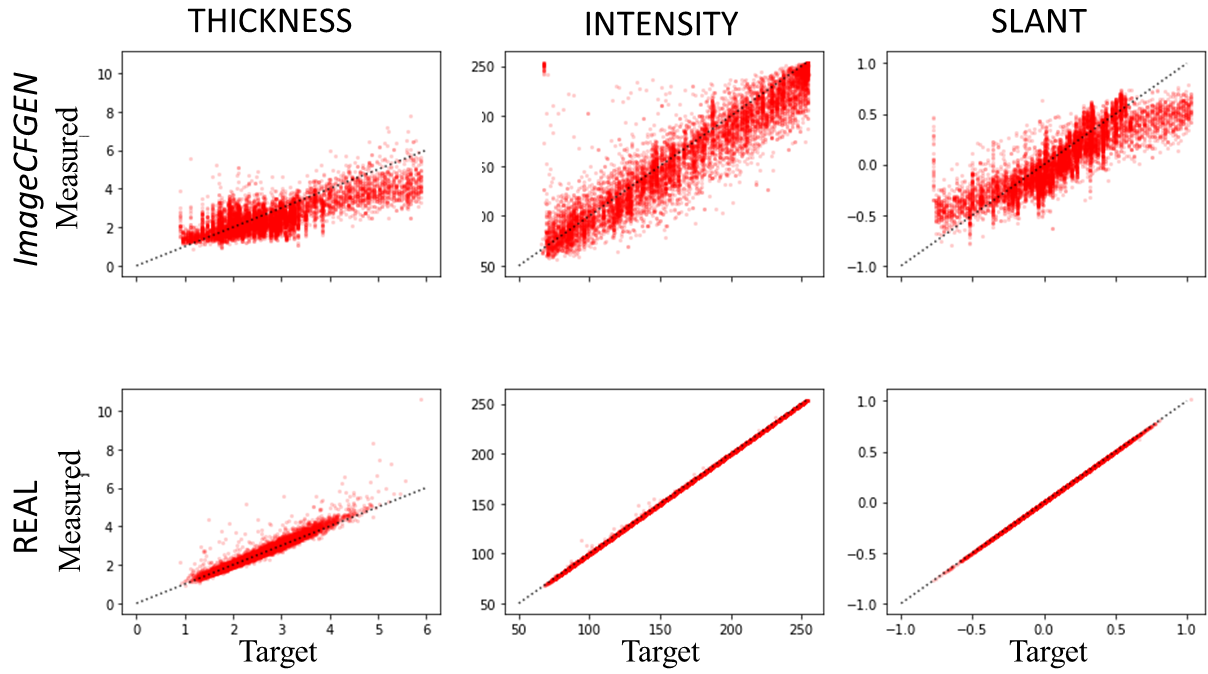}
    \caption{\textbf{Morpho-MNIST CFs.} For each attribute, Target ($x$ axis) is the desired value for the CF and Measured ($y$ axis) is the attribute value obtained from the generated counterfactual using a formula provided in the Morpho-MNIST dataset~\cite{pawlowski2020deep}. We compare between ground-truth CFs  and CSALI-generated CFs. In an ideal scenario (real samples), points should lie along the $y = x$ line. 
    }
 \label{mnistline}
\end{figure}

To quantify these observations, we randomly sample hundred values for \textit{thickness}, \textit{intensity} and \textit{slant} attributes and generate corresponding CFs for each test image. Fig \ref{mnistline} plots the target ground-truth attribute vs the measured  attribute in the CF image (measured using a formula from Morpho-MNIST~\cite{pawlowski2020deep}), and shows that all attributes are on an average clustered along the reference line of y = x with some variability. We quantify this variability in Table \ref{table1} using median absolute error, by comparing the CFs generated using \ourmethod\ vs those using the DeepSCM method~\cite{pawlowski2020deep} (we choose Median Absolute Error to avoid skewed measurements due to outliers). 
\ourmethod\ and DeepSCM are comparable on attributes of \textit{thickness} and \textit{intensity}, showing the validity of our CFs. We could not compare \textit{slant} since \cite{pawlowski2020deep} do not use \textit{slant} in their studies.
\begin{table}[!htb]
\centering
\begin{tabular}{l|l|l}

        & $do$ (thickness)  & $do$ (intensity) \\
        \hline
        \rule{0pt}{2.6ex}\ourmethod     & 0.408 $\pm$ 0.004          & \textbf{10.740 $\pm$ 0.161} \\
DeepSCM & \textbf{0.253 $\pm$ 0.002} & 12.519 $\pm$ 0.105        
\end{tabular}
\caption{\textbf{Median Abs. Error: \ourmethod\ and DeepSCM.} Lower is better.}

\label{table1}
\end{table}

\subsection{Quality of \ourmethod\ CFs on CelebA}
We now evaluate \ourmethod\ on CelebA by comparing the quality of the generated CF images against the vanilla ALI model and the DeepSCM model \cite{pawlowski2020deep}. 

\noindent \textbf{Generation Quality.} We show real images and their corresponding reconstructions using ALI and CSALI in rows (I), (II) and (IV) of Fig \ref{celebaablation} in Appendix K. While the reconstructions of both ALI and CSALI are not perfect, those of CSALI are significantly better than ALI. Moreover, they capture majority of the high-level features of the real images like \textit{Hair Color}, \textit{Smiling}, \textit{Pale}, etc.
We quantify this in Table \ref{table2} using Fréchet Inception Distance \cite{heusel2017gans} (FID) score for generated images, Mean Squared Error (MSE) between the real images $\vecx$ and reconstructed images $\vecx_r$ and Mean Absolute Error (MAE) between the real latent variable $\vecz$ and the encoder's approximation of the latent variable $\vecz_r$. We randomly sample 10k generated images to calculate FID scores, and randomly sample 10k real images for which we generate reconstructions and report the MSE and MAE metrics. We report these  metrics on the baseline model of ALI as well. We observe a significant improvement in generation quality after using a style-based generator and significant improvements in reconstruction with the cyclic cost minimization algorithm (refer to Appendix K for ablation study). We report MSE on images and MAE on latent space variables since the cyclic cost minimization algorithm \cite{dogan2020semi} uses these metrics to improve reconstructions. 

\begin{table}[!htb]
\centering
\begin{tabular}{l|l|l|l}

        & FID  & MSE ($x$,$x_r$) & MAE ($z, z_r$) \\
        \hline
        \rule{0pt}{2.6ex}ALI     & 67.133  & 0.177 & 1.938 \\
CSALI & \textbf{21.269} & \textbf{0.103} & \textbf{0.940}        
\end{tabular}
\caption{\textbf{FID, MSE and MAE scores for ALI and Cyclic Style ALI (CSALI).} Lower is better.}
\label{table2}
\end{table}

\noindent \textbf{Counterfactual Quality.} We contrast the quality and validity of the CFs from \ourmethod\ with the CFs from DeepSCM \cite{pawlowski2020deep} (refer to Appendix H for our implementation of DeepSCM on CelebA).
To generate the counterfactual images, we intervene on the attributes of \textit{Black Hair}, \textit{Blond Hair}, \textit{Brown Hair}, \textit{Pale} and \textit{Bangs}. 
We observe in Figure \ref{icfgen_dscm} that the CFs from \ourmethod\ are qualitatively better than those from DeepSCM. \ourmethod\ CFs successfully change the hair color and skin color, in addition to adding bangs to the face (refer to Figure \ref{celebacnt} in Appendix J for more \ourmethod\ CFs and Appendix I for more comparisons with DeepSCM). In contrast, the CFs from DeepSCM only partially change the hair color and the skin color in columns (a) through (f) and fail to add bangs in column (g). 
\begin{figure}[!htb]
\centering
\includegraphics[width = \columnwidth]{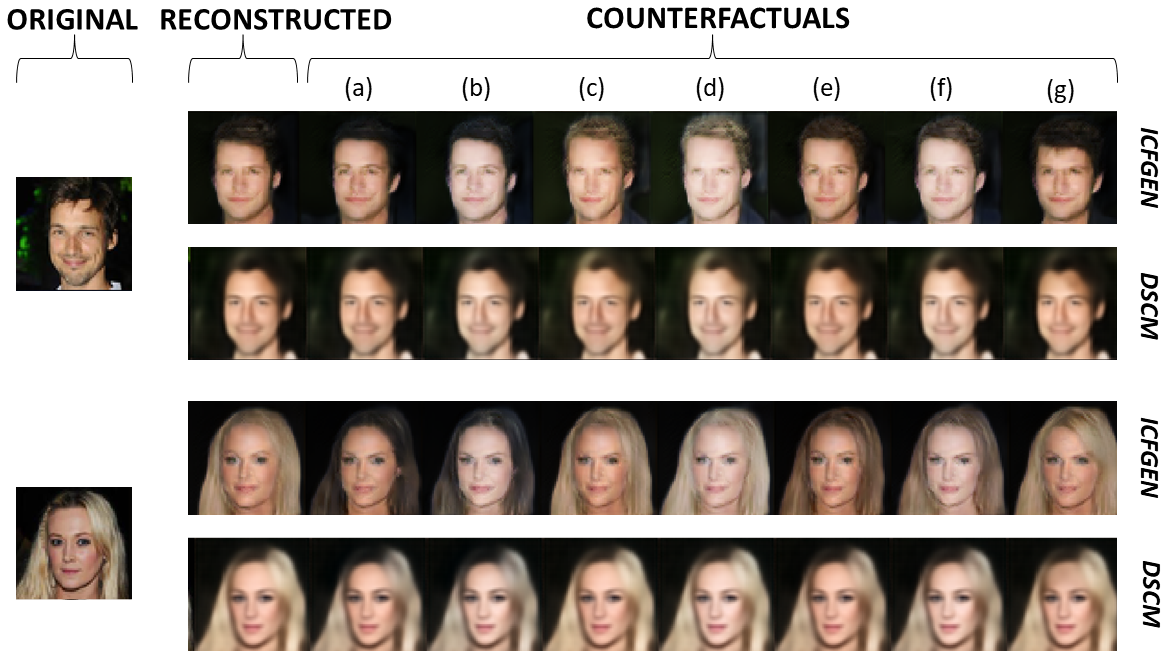}
\caption{\textbf{\ourmethod\ and DeepSCM Counterfactuals.} (a) denotes $do$ (black hair = 1) and (b) denotes $do$ (black hair = 1, pale =1). Similarly (c) denotes $do$ (blond hair = 1); (d) denotes $do$ (blond hair = 1, pale = 1); (e) denotes $do$ (brown hair = 1); (hf denotes $do$ (brown hair = 1, pale = 1); and (g) denotes $do$ (bangs = 1).}
\label{icfgen_dscm}
\end{figure}

We also perform a human evaluation experiment of the generated counterfactuals in Appendix L, which showed that the distribution of counterfactuals is identical to the distribution of their corresponding base images.

\subsection{Bias Evaluation \& Explaining a Classifier}

We train a classifier on the CelebA dataset to predict attractiveness of an image w.r.t. an attribute (architecture and training details in Appendix M). We then use the generated CFs to identify biases in the classifier.
We sample 10k points from the test set and generate seven CFs for each of them as shown in Fig \ref{icfgen_dscm} for different attributes. We only consider those images for which the corresponding attribute was absent in order to avoid redundancies. For instance, we filter out CF (c) of the second sample from Fig \ref{icfgen_dscm} since blond hair was already present in the base image. We then provide the generated CFs along with the base (reconstructed) image to the attractive classifier and compare their labels.
As a baseline comparison, we also pass images with affine transformations like horizontal flip and increasing brightness to the classifier. 
We treat the classifier's outputs as the probability of being labeled as attractive. Fig.~\ref{celebafair} shows these probabilities for the base image, affine transformations, and the CFs. If the classifier is fair w.r.t. these attributes, all points should be clustered along the y=x line. 
\begin{figure}[!htb]
\centering
  \includegraphics[width = \columnwidth]{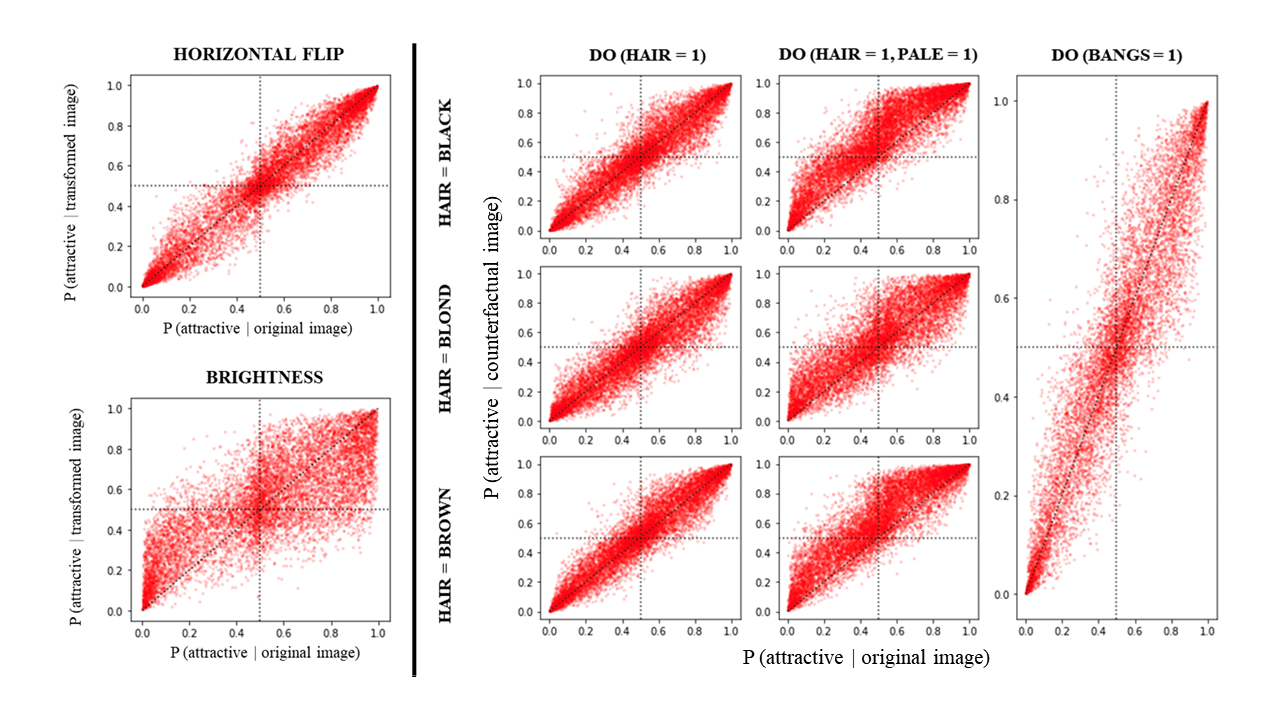}
  \caption{\textbf{Fairness Analysis.} Affine transformations \textit{(left)}, CFs \textit{(right)}. Each point in a scatter plot is a pair of a base image and its corresponding CF image. In the ideal case, all points should lie along the y = x line. To analyze the figures, divide the scatter plot into four quadrants formed by lines x = 0.5 and y = 0.5. Any point in the top left quadrant signifies that the attractive label was changed from 0 to 1 and vice-versa for the bottom right quadrant. 
}
  \label{celebafair}
\end{figure}

For the CF plots, for $do$ (black hair =1, pale = 1) and $do$ (brown = 1, pale = 1) almost all points are above the reference line, suggesting that the probability of being classified as attractive increases after applying these interventions. Not only does the probability increase, for $do$ (black hair = 1, pale = 1), 18\% of the CF labels are flipped from the base ones, and \textit{94\%} of these labels are changed from not attractive to attractive. In case of $do$ (brown hair = 1, pale = 1),  19\% of the CF labels are flipped and \textit{94\%} of the labels are flipped from not attractive to attractive. For the CF $do$ (blond hair = 1, pale = 1), 16\% of the labels are flipped and 74\% of the labels are flipped from attractive to not attractive. In comparison, the affine transformations are unable to provide a clear picture of bias. For horizontal flip, the points are equally distributed on both sides of the reference line y = x. In the case of brightness,  there is more variation. 

In Table \ref{estimates}, we quantify these observations using Eqn \ref{bias}. Our metric for bias measurement gives an overall estimate of the bias in  classifier, and provides an interpretable and uniform scale to compare biases among different classifiers. The reported bias values reflect the observations from Fig \ref{celebafair}. We observe that the CF of $do$ (brown hair = 1, pale = 1) has the highest positive bias amongst all CFs and affine transformations, i.e. the classifier is biased towards labeling these CFs as more attractive in contrast to the base image. Using CFs, we are able to detect other significant biases towards setting skin color as $pale=1$ for all hair colors (black, blond and brown). In contrast, using the baseline transformations, we are unable to detect skin color bias in the classifier since the calculated bias values are negligible. 
\begin{table}[tb]
\centering
\begin{tabular}{l|l|l|l}
& p($a_r \neq a_c$) & p(0 $\to$ 1) & bias \\
\hline
\rule{0pt}{2.6ex}horizontal\_flip             &      0.073             &              0.436                      &   -0.009   \\
brightness  &     0.192    &      0.498                               &   -0.001   \\
\hdashline
\rule{0pt}{2.6ex}black\_h &    0.103   &  0.586             &   0.018  \\
black\_h, pale  &     0.180            &            0.937                      &    \textbf{0.158} \\
blond\_h    &       0.115            &      0.413   &   - 0.02   \\
blond\_h, pale  &      0.155           &     0.738                                &    \textbf{0.073}  \\
brown\_h     &         0.099          &       0.704                              &    0.041  \\
brown\_h, pale  &      0.186           &               0.942                      &  \textbf{0.164}    \\
bangs    &          0.106         &                  0.526                   & 0.005   
\end{tabular}
\caption{\textbf{Bias Estimation.} Bias values above a threshold of 5\% are considered significant.}
\label{estimates}
\end{table}

\begin{figure}[!htb] 
\centering
\includegraphics[width = 0.95\columnwidth]{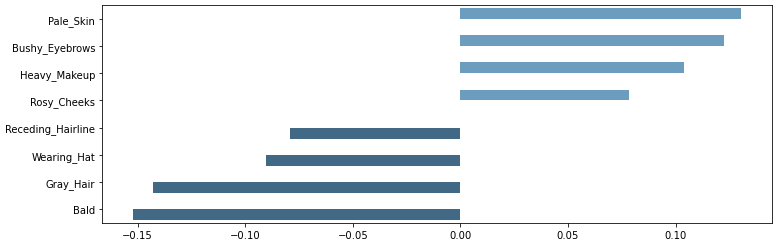}
\caption{\textbf{Explaining a Classifier.} Attribute ranking of top 4 positive and top 4 negative influential attributes.}
\label{attrrank}
\end{figure}

We also use CFs to explain the classifier's decisions for predicting attractiveness of a face. Fig \ref{attrrank} shows the top 4 positive and top 4 negative influences for classifying a face as attractive. We can see that \textit{Pale} skin is the top attribute that contributes to the classifier predicting a face as attractive, while \textit{Bald} is the top attribute that contributes to the classifier prediction of not attractive.

\subsection{Bias Mitigation for a Classifier}
Finally, using Eqn.~\ref{eq:bias-mitigation}, we employ the generated CFs to remove the identified biases in the \textit{Attractive} classifier on $do$ (black = 1, pale = 1), $do$ (blond = 1, pale = 1) and $do$ (brown = 1, pale = 1) . 
The CF-regularized classifier reports a bias score of 0.032 for black hair and pale (against 0.159 for the original classifier) and 0.012 for brown hair and pale (against 0.154 for the original classifier). Also, the reduced biases are no longer significant, without reducing the accuracy (82.3\% versus 82.6\%). Details are in Appendix N. 

Additionally, our method can also be utilized in the setting where the attributes are connected in a complex causal graph structure, unlike \cite{denton2019detecting, joo2020gender}. We conduct a similar fairness and explanation analysis for a \textit{Attractive} classifier in Appendix O, where \textit{Young} affects other visible attributes like \textit{Gray hair}.

\section{Conclusion}

We propose a framework \ourmethod\ for generating counterfactuals based on an underlying SCM, utilizing the  generative capacity of GANs. 
We demonstrate how the counterfactuals can be used to evaluate and mitigate a classifier's biases and explain its decisions.  
That said, we acknowledge two limitations , 1) Our CF generation method relies on accurate knowledge of the causal graph; 2) It uses a statistical model that can have have unknown failure modes in generating meaningful counterfactuals.
Therefore, this work should be considered as a prototype early work in generating counterfactuals, and is not suitable for deployment.

\bibliography{ref.bib}

\clearpage
\appendix

\noindent In this appendix, we discuss the following details, which could not be included in the main paper owing to space constraints.

\section{\ourmethod Algorithm}

As described in Section~\ref{subsec_scm_over_attribs}, our method involves two components: learning the Attribute-SCM ($\mathcal{M}_a$) that models relationships between attributes,  and learning the generator ($\mathcal{M}_x$) that produces counterfactual images given the attributes.

\subsection{Learning the Attribute-SCM}

We assume that the causal graph structure is provided by the user, but the causal functions relating the attributes are not available. For each attribute $a_i$, the graph entails a set of attributes $P_{a_i}$ that cause the attribute (its \textit{parents} in the graph structure) that should be included while estimating the value of the attribute. It leads to the following generating equation for each attribute, 
\begin{equation}
    a_i = g_i(P_{a_i},\epsilon_i)
\end{equation}
where $P_{a_i}$ are the parents of attribute $a_i$ and $\epsilon_i$ is independent noise. The goal is to learn the unknown function $g_i$ for each attribute. Thus, given $n$ attributes, we obtain a set of $n$ equations, each of which can be independently estimated using Maximum Likelihood Estimation from Bayesian networks~\cite{ding2010probabilistic}. 

 \begin{algorithm}[h]
\small
\SetAlgoLined
\KwIn{Input data-attribute pair ($\vecx^*$, $\veca^*$); Indices of sensitive attributes to modify $K$; Modified attribute values $\veca^{**}_K$, Attribute SCM $\mathcal{M}_a = \{g_i(P_{a^*_i}, \hat{\epsilon}_i)\} \forall a^*_i \in \veca^*, P_{a^*_i} \subseteq \veca^*$ where $P_{a^*_i}$ denotes parents of $a^*_i$ and $\hat{\epsilon}_i$ is estimated from data, Encoder $E$, Generator $G$}
\KwOut{$\vecx_c$}
{\bf{Step 1: Abduction}}\\ 
$\vecz = E(\vecx^*, \veca^*)$\\
{\bf{Step 2: Action}}\\ 
Remove all incoming arrows to $\veca_K$ in $\mathcal{M}_a$ to yield $\mathcal{M}_a'$ \\
$a_i \rightarrow a^{**}_i \quad  \forall i \in K, a^{**}_i \in \veca^{**}_K$ \tcc{Intervene on sensitive attributes and set to modified attribute values}
{\bf{Step 3: Prediction}}\\ 
$D = desc(\veca_K)$ \tcc{descendants of $\veca_K$}
$currset = children(\veca_K)$ \\
$changed[\textbf{a}^*] = False$ \\
\While{currset is not empty}{
    $cs=\emptyset$; \\
    \ForEach{$a_j \in currset$}{
        \tcc{Proceed only if the values of all parents of $a_j$ are already changed}
        \If{not $any_l(P^{(l)}_{a_j} \in D$ \& $changed[P^{(l)}_{a_j}]==False)$)}{
        $a^{**}_j \rightarrow g_j(P_{a_j^{**}}, \hat{\epsilon_j})$\\
        $cs.append(children(a_j))$\\
        $changed[a_j] = True$
        }
    }
    $currset=cs$ \tcc{If there exist children of the current nodes, repeat while loop}
} 
$\veca_c = \veca^{**}_K \cup \veca^{**}_{-K}$ \tcc{$\veca^{**}_{-K}$ denotes the final values for all non-sensitive attributes}
$\vecx_c = G(\vecz, \veca_c)$
\caption{\ourmethod} 
\label{algorithm1}
\end{algorithm}

\subsection{Generating Counterfactuals}
Algorithm \ref{algorithm1} below shows how \ourmethod\ is implemented, in extension to the discussion in Sec 3.2 of the main paper. Here we assume that the Attribute-SCM, the Encoder $E$, and Generator $G$ are pre-learnt and provided as input to the algorithm. 
$(\vecx^*,\veca^*)$ denotes the input value of the image and the attributes. Note that the Attribute SCM operations (Steps 1-3) are based on the formal procedure to generate counterfactuals from ~\cite{pearl2009book}.  

In Step 1 (Abduction), the encoder uses the input data $(\vecx^*,\veca^*)$ to create the latent vector $\vecz$.

In Step 2 (Action), given the indices $K$ of sensitive attributes, the goal is to remove the  dependence of these sensitive attributes on any other attribute and then change their value to the desired value. To do so, the causal graph of Attribute-SCM is modified to  remove all incoming arrows to $\veca_K$. That is, the structural equations for each $a_i \in \veca_K$ changes from $a_i = g_i(P_{a_i}, \epsilon_i)$ to $a_i = \epsilon'_i$, yielding a modified Attribute-SCM $\mathcal{M}'_a$. Then each $a_i \in \veca_K$ is set to the desired modified value $a_i^{**}$ (Line 5; we do this without changing its parents or other variables, hence it is called an \textit{intervention}).

In Step 3 (Prediction), we propagate the change in sensitive attributes to all attributes caused by them. That is, values of all descendants of sensitive attributes $\veca_K$ in the causal graph of Attribute-SCM $\mathcal{M}'_a$ are changed, based on the modified value of the sensitive attributes from Step 2. We do so level-by-level: first changing the values of children of $\veca_K$ based on the modified value of $\veca_K$, then the values of children of children of $\veca_K$, and so on. Note that line 13 ensures that an attribute's value is changed only when all its parents' values have been modified. The corresponding equation for each descendant $a_j$ is given in Line 14,
\begin{equation}
    a^{**}_j \rightarrow g_j(P_{a^{**}_j}, \hat{\epsilon_j}) 
\end{equation}
where $P_{a^{**}_j}$ refers to the modified values of the parents of $a_j$, $g_j$ is the pre-learnt Attribute-SCM function,  and $\hat{\epsilon_j}$ is the estimated error term from the original attribute data (i.e., it is the same error value that satisfies  $a^{*}_j \rightarrow g_j(P_{a^{*}_j}, \hat{\epsilon_j})$). Finally, the updated values of all attributes are used to create $\veca_c$, the counterfactual attribute vector. This counterfactual attribute vector is then combined with latent $z$ from Step 1 and provided as input to the generator G, to obtain the counterfactual image $\vecx_c$ (Line 22).

\section{Counterfactual Generation Proof}
\begin{restatable}{proposition}{cfproof} \label{thm:valid-cf}
    Assume that the true SCM $\mathcal{M}$ belongs to a class of SCMs where the structural equations can be identified uniquely given the causal graph. Further, assume that the Encoder $E$, Generator $G$ and Attribute SCM are optimal and thus  correspond to the true structural equations of the SCM $\mathcal{M}$. Then Equation~\ref{counterfactual} generates a valid counterfactual image for any given input $(\vecx, \veca)$ and the requested modified attributes $\veca'_k$.
\end{restatable}
\begin{proof}
    Let $\mathcal{M} = \{\mathcal{M}_x, \mathcal{M}_a\}$ be the true SCM  that generates the data $(\vecx, \veca)$. Let $\veca_k \subset \veca$ be the attributes you want to intervene on. Let $\veca_{-k} = \veca \setminus \veca_k$ be the remaining attributes. The corresponding equations for $\mathcal{M}$ are:
    \begin{equation} \label{eq:scm}
        \begin{split}
        a_{i} &:= g_i ( p_{a_i}, \epsilon_i) ,\quad \forall i = 1..n\\
        \vecx &:= g(\veca_k, \veca_{-k}, \epsilon)
        \end{split}
    \end{equation}
    where $\epsilon$ and $\epsilon_i$ refer to independent noise, g and $g_i$ refer to structural assignments of the SCMs $\mathcal{M}_x$ and $\mathcal{M}_a$ respectively while $p_{a_i}$ refers to parent attributes of $a_i$. 
    Given an input $(\vecx, a_k, a_{-k})$, we generate a counterfactual using $\mathcal{M}$ and show that it is equivalent to Equation~\ref{counterfactual}.

\noindent \textbf{Abduction:} Infer $\epsilon$ for $\vecx$ and $\epsilon_i$ for all $a_i$ from Equation~\ref{eq:scm} using the following equations:
    \begin{equation}
        \begin{split}
        \hat{\epsilon_i} &:= g_i^{-1}(a_i, p_{a_i}) ,\quad \forall i = 1..n \\
            \hat{\epsilon} &:= g^{-1}(\vecx, \veca_k, \veca_{-k})
 \end{split}
    \end{equation}
    
\noindent \textbf{Action:} Intervene on all attributes $\veca_k$ by setting them to the requested modified attributes $\veca'_k$.

    \begin{equation}
        a_i \rightarrow a'_i \quad \forall a_i \in \veca_k \quad \textrm{and} \quad \forall a'_i \in \veca'_k
    \end{equation}
    
\noindent \textbf{Prediction:}
    The following equation then changes the values of all descendants of $\veca_k$.
    \begin{equation} \label{prediction}
      desc(a_i) \rightarrow g_i( p_{desc(a_i)}, \hat{\epsilon_i}) \quad \forall a_i \in \veca_k 
    \end{equation}
where $desc(a_i)$ are descendants of $a_i$ and $p_{desc(a_i)}$ are parents of the descendants of $a_i$, $\forall a_i \in \veca_k$. Let $\veca_c = \veca'_k \cup \veca'_{-k}$ where $\veca'_{-k}$ are (possibly) modified values of the other attributes according to Equation~\ref{prediction}. 
Therefore, the counterfactual of $\vecx$, $\vecx_c$ can be generated as:
\begin{equation}
    \vecx_c := g(\veca_c, \hat{\epsilon})
\end{equation}

We now show that Equation~\ref{counterfactual} produces the same $\vecx_c$. By the assumption in the theorem statement, the Attribute-SCM corresponds to the structural assignments $\{g_i, g_i^{-1}\}$, $\forall i = 1..n$ of SCM $\mathcal{M}_a$ while the Generator $G$ learns the structural assignment $g$ and the Encoder $E$ learns $g^{-1}$ of the SCM $\mathcal{M}_x$. Hence, the Attribute-SCM, Generator and Encoder learn the combined SCM $\mathcal{M}$.

When the SCM assignments learned by the Attribute-SCM are optimal, i.e. Attribute-SCM = $\mathcal{M}_a$ then:
$$ \hat{a}_{c} = a_{c} $$
Similarly, under optimal Generator, $G = g$ and $E= g^{-1}$:
\begin{equation}
    \begin{split}
        \vecx_c &= g(\veca_c, g^{-1}(\vecx, \veca_k, \veca_{-k})) \\
        &= G(\veca_c, E(\vecx, \veca)) \quad (\textrm{as} \quad \veca = \veca_k \cup \veca_{-k}) \\
\end{split}
\end{equation}
which is the same as Equation~\ref{counterfactual}.
\end{proof}

\section{Counterfactual Fairness Proof}

\begin{restatable}{definition}{cfdef}{Counterfactual Fairness from ~\cite{kusner2017counterfactual}.} 
    Let $\mathcal A$ be the set of attributes, comprising of sensitive attributes $\mathcal A_S \subseteq \mathcal A$ and other non-sensitive attributes $\mathcal A_N$. The classifier $\hat{f}$ is counterfactually fair if under any context $\vecX=\vecx$ and $\vecA=\veca$, changing the value of the sensitive features to $\vecA_S \leftarrow a_s'$ counterfactually does not change the classifier's output distribution $Y$. 
    \begin{equation}\begin{medsize}
        \begin{aligned}
        P(Y_{A_S\leftarrow a_s}=y|\vecX=\vecx, \vecA_S=\veca_s, \vecA_N=\veca_N) \\
        =P(Y_{A_S\leftarrow a'_s}=y|\vecX=\vecx, \vecA_S=\veca_s, \vecA_N=\veca_N) 
    \end{aligned}
    \end{medsize}
    \end{equation}
    for all $y$ and for any value $\veca_s'$ attainable by $\vecA_S$.
\end{restatable}

\begin{restatable}{proposition}{cfgenfair}
    Under the assumptions of Proposition~\ref{thm:valid-cf} for the  encoder $E$, generator $G$, and Attribute SCM $\mathcal{M}_a$, a classifier $\hat{f}(\vecX):\mathcal{X} \to \mathcal{Y}$ that satisfies zero bias according to Equation \ref{bias} is counterfactually fair with respect to $\mathcal{M}$. 
\end{restatable}

\begin{proof}
To evaluate fairness, we need to reason about the output $Y = y$ of the classifier $\hat{f}$. Therefore, we add a functional equation to the SCM $\mathcal{M}$ from Equation~\ref{eq:scm}.
 \begin{equation} 
        \begin{split}
        a_{i} &:= g_i ( p_{a_i}, \epsilon_i) ,\quad \forall i = 1..n\\
        \vecx &:= g(\veca_k, \veca_{-k}, \epsilon) \\
        y & \leftarrow \hat{f}(\vecx)
        \end{split}
    \end{equation}
    where $\epsilon$ and $\epsilon_i$ are independent errors as defined in Equation~\ref{eq:scm}, and $P_{\veca_i}$ refers to the parent attributes of an attribute $a_i$ as per the Attribute-SCM. The SCM equation for $y$ does not include an noise term since the ML classifier $\hat{f}$ is a deterministic function of $\vecX$. 
    In the above equations, the attributes $\veca$ are separated into two subsets: $\veca_k$ are the attributes specified to be changed in a counterfactual and $\veca_{-k}$ refers to all other attributes.
    
    Based on this SCM, we now generate a counterfactual $y_{a_k\leftarrow a'_k}$ for an arbitrary new value $a'_{k}$. Using the \textbf{Prediction} step, the counterfactual output label $y_c$ for an input $(\vecx, \veca)$ is given by:  $y_c = \hat{f}(\vecX= \vecx_c)$.  From Theorem 1, under optimality of the encoder $E$, generator $G$ and learned functions $g_i$, we know that $\vecx_c$ generated by the following equation is a valid counterfactual for an input $(\vecx, \veca)$,
    \begin{equation}\begin{medsize}
    \begin{aligned}
         \vecx_c &= G(E(\vecx, \veca), \veca_c)  \\
         &= X_{\vecA_k\leftarrow a_k'}|(\vecX=\vecx, \vecA_k=\veca_k, \vecA_{-k} = \veca_{-k})
         \end{aligned}
        \end{medsize}
    \end{equation}
where $\veca_c$ represents the modified values of the attributes under the action $\vecA_k\leftarrow a_k'$. 
Therefore, $y_{a_k\leftarrow a'_k}|(\vecX=\vecx, \vecA_k=\veca_k, \vecA_{-k} = \veca_{-k})$ is given by $y_c=\hat{f}(\vecx_c)$.

Using the above result, we now show that the bias term from Equation~\ref{bias} and \ref{bias-simple} reduces to the counterfactual fairness definition from Definition 1, 
\begin{equation}
\begin{medsize}
\begin{aligned}
 P&(y_r = 0, y_c = 1) - P(y_r = 1, y_c = 0)\\
    &= [P(y_r = 0, y_c = 1) + P(y_r = 1, y_c = 1)] - \\
    & \text{ \ \ \ \ } [P(y_r = 1, y_c = 0) + (P(y_r =1, y_c = 1)] \\
    &= P(y_c=1) - P(y_r=1)\\
    &= [P(Y_{A_k\leftarrow a'_k}=1|\vecX=\vecx, \vecA_k=\veca_k, \vecA_{-k} = \veca_{-k}) -\\
     & \text{ \ \ \ \ } P(Y_{A_k\leftarrow a_k}=1|\vecX=\vecx, \vecA_k=\veca_k, \vecA_{-k} = \veca_{-k})] \\
     &= [ P(Y_{A_S\leftarrow a'_s}=1|\vecX=\vecx, \vecA_S=\veca_s, \vecA_N = \veca_n) -\\
     & \text{ \ \ \ \ } P(Y_{A_S\leftarrow a_s}=1|\vecX=\vecx, \vecA_S=\veca_s, \vecA_N = \veca_n)]
      \end{aligned}
    \end{medsize}
\end{equation}
where the second equality is since $y_r, y_c \in \{0,1\}$, the third equality is since the reconstructed $y_r$ is the output prediction when $\vecA=\veca$, and the last equality is $\vecA_k$ being replaced by the sensitive attributes $\vecA_S$ and $\vecA_{-k}$ being replaced by $\vecA_N$. We can prove a similar result for $y_c=0$. Hence, when bias term is zero, the ML model $\hat{f}$ satisfies counterfactual fairness (Definition 1). 
\end{proof}

\section{Architecture Details}

Here we provide the architecture details for the base Adversarially Learned Inference (ALI) model trained  on Morpho-MNIST and CelebA datasets. For Cyclic Style ALI, we replace the ALI generator with the style-based generator architecture in \cite{karras2019style}. We however do not use progressive growing or other regularization strategies suggested in \cite{karras2019style,karras2020analyzing} while training our model.
For details on the style-based generation architecture, please refer \cite{karras2019style}.

Overall, the architectures and hyperparameters are similar to the ones used by \cite{dumoulin2016adversarially}, with minor variations.
Instead of using the Embedding network from the original paper, the attributes are directly passed on to the Encoder, Generator and Discriminator. We found that this enabled better conditional generation in our experiments. All experiments were implemented using  Keras 2.3.0 \cite{chollet2015keras} with Tensorflow 1.14.0 \cite{abadi2016tensorflow}. All models were trained using a Nvidia Tesla P100 GPU. 

\subsection{Morpho-MNIST}
Tables \ref{MGen}, \ref{MEnc} and \ref{MDisc} show the Generator, Encoder and Discriminator architectures respectively for generating Morpho-MNIST counterfactuals. Conv2D refers to Convolution 2D layers, Conv2DT refers to Transpose Convoloution layers, F refers to number of filters, K refers to kernel width and height, S refers to strides, BN refers to Batch Normalization, D refers to dropout probability and A refers to the activation function. LReLU denotes the Leaky ReLU activation function. 
We use the Adam optimizer~\cite{kingma2014adam} with a learning rate of $10^{-4}$, $\beta_1 = 0.5$ and a batch size of 100 for training the model. For the LeakyReLU activations, $\alpha = 0.1$. The model converges in approximately 30k iterations. All weights are initialized using the Keras truncated normal initializer with $mean = 0.0$ and $stddev = 0.01$. All biases are initialized with zeros. 

\begin{table}[!htb]
\centering

\begin{tabular}{lllllll}
\hline
Layer   & F   & K     & S   & BN & D   & A \\
\hline \\
Conv2DT & 256 & (4,4) & (1,1) & Y  & 0.0 & LReLU  \\
Conv2DT & 128 & (4,4) & (2,2) & Y  & 0.0 & LReLU  \\
Conv2DT & 64 & (4,4) & (1,1) & Y  & 0.0 & LReLU  \\
Conv2DT & 32 & (4,4) & (2,2) & Y  & 0.0 & LReLU  \\
Conv2DT & 32  & (1,1) & (1,1) & Y  & 0.0 & LReLU  \\
Conv2D & 1   & (1,1) & (1,1) & Y  & 0.0 & Sigmoid \\
\hline
\end{tabular}
\caption{\textbf{Architecture for Morpho-MNIST Generator}}
\label{MGen}
\end{table}

\begin{table}[!htb]
\centering

\begin{tabular}{lllllll}
\hline
Layer  & F   & K     & S     & BN & D   & A \\
\hline
Conv2D & 32  & (5,5) & (1,1) & Y  & 0.0 & LReLU  \\
Conv2D & 64 & (4,4) & (2,2) & Y  & 0.0 & LReLU  \\
Conv2D & 128 & (4,4) & (1,1) & Y  & 0.0 & LReLU  \\
Conv2D & 256 & (4,4) & (2,2) & Y  & 0.0 & LReLU  \\
Conv2D & 512 & (3,3) & (1,1) & Y  & 0.0 & LReLU  \\
Conv2D & 512 & (1,1) & (1,1) & Y  & 0.0 & Linear  \\
\hline  
\end{tabular}
\caption{\textbf{Architecture for Morpho-MNIST Encoder}}
\label{MEnc}
\end{table}

\begin{table}[!htb]
\centering

\begin{tabular}{lllllll}
\hline
Layer   & F    & K     & S     & BN & D   & A \\
\hline
$D_z$    &      &       &       &    &     &            \\
Conv2D  & 512 & (1,1) & (1,1) & N  & 0.2 & LReLU  \\
Conv2D  & 512 & (1,1) & (1,1) & N  & 0.5 & LReLU  \\
$D_x$    &      &       &       &    &     &            \\
Conv2D  & 32   & (5,5) & (1,1) & N  & 0.2 & LReLU  \\
Conv2D  & 64  & (4,4) & (2,2) & Y  & 0.2 & LReLU  \\
Conv2D  & 128  & (4,4) & (1,1) & Y  & 0.5 & LReLU  \\
Conv2D  & 256  & (4,4) & (2,2) & Y  & 0.5 & LReLU  \\
Conv2D  & 512  & (3,3) & (1,1) & Y  & 0.5 & LReLU  \\
$D_{x\_z}$ &      &       &       &    &     &            \\
Conv2D  & 1024 & (1,1) & (1,1) & N  & 0.2 & LReLU  \\
Conv2D  & 1024 & (1,1) & (1,1) & N  & 0.2 & LReLU  \\
Conv2D  & 1    & (1,1) & (1,1) & N  & 0.5 & Sigmoid\\
\hline   
\end{tabular}
\caption{\textbf{Architecture for Morpho-MNIST Discriminator.} $D_x$, $D_z$ and $D_{xz}$ are the discriminator components to process the image $x$, the latent variable $z$ and the output of $D_x$ and $D_z$ concatenated, respectively.}
\label{MDisc}
\end{table}

\subsection{CelebA}
Tables \ref{CGen}, \ref{CEnc} and \ref{CDisc} show the Generator, Encoder and Discriminator architectures respectively for generating CelebA counterfactuals.
\begin{table}[!htb]
\centering

\begin{tabular}{lllllll}
\hline
Layer   & F   & K     & S   & BN & D   & A \\
\hline \\
Conv2DT & 512 & (4,4) & (1,1) & Y  & 0.0 & LReLU  \\
Conv2DT & 256 & (7,7) & (2,2) & Y  & 0.0 & LReLU  \\
Conv2DT & 256 & (5,5) & (2,2) & Y  & 0.0 & LReLU  \\
Conv2DT & 128 & (7,7) & (2,2) & Y  & 0.0 & LReLU  \\
Conv2DT & 64  & (2,2) & (1,1) & Y  & 0.0 & LReLU  \\
Conv2D & 3   & (1,1) & (1,1) & Y  & 0.0 & Sigmoid\\
\hline 
\end{tabular}
\caption{\textbf{Architecture for CelebA Generator}}
\label{CGen}
\end{table}

\begin{table}[!htb]
\centering

\begin{tabular}{lllllll}
\hline
Layer  & F   & K     & S     & BN & D   & A \\
\hline
Conv2D & 64  & (2,2) & (1,1) & Y  & 0.0 & LReLU  \\
Conv2D & 128 & (7,7) & (2,2) & Y  & 0.0 & LReLU  \\
Conv2D & 256 & (5,5) & (2,2) & Y  & 0.0 & LReLU  \\
Conv2D & 256 & (7,7) & (2,2) & Y  & 0.0 & LReLU  \\
Conv2D & 512 & (4,4) & (1,1) & Y  & 0.0 & LReLU  \\
Conv2D & 512 & (1,1) & (1,1) & Y  & 0.0 & Linear \\
\hline   
\end{tabular}
\caption{\textbf{Architecture for CelebA Encoder}}
\label{CEnc}
\end{table}

\begin{table}[!htb]
\centering

\begin{tabular}{lllllll}
\hline
Layer   & F    & K     & S     & BN & D   & A \\
\hline
$D_z$    &      &       &       &    &     &            \\
Conv2D  & 1024 & (1,1) & (1,1) & N  & 0.2 & LReLU  \\
Conv2D  & 1024 & (1,1) & (1,1) & N  & 0.2 & LReLU  \\
$D_x$    &      &       &       &    &     &            \\
Conv2D  & 64   & (2,2) & (1,1) & N  & 0.2 & LReLU  \\
Conv2D  & 128  & (7,7) & (2,2) & Y  & 0.2 & LReLU  \\
Conv2D  & 256  & (5,5) & (2,2) & Y  & 0.2 & LReLU  \\
Conv2D  & 256  & (7,7) & (2,2) & Y  & 0.2 & LReLU  \\
Conv2D  & 512  & (4,4) & (1,1) & Y  & 0.2 & LReLU  \\
$D_{x\_z}$ &      &       &       &    &     &            \\
Conv2D  & 2048 & (1,1) & (1,1) & N  & 0.2 & LReLU  \\
Conv2D  & 2048 & (1,1) & (1,1) & N  & 0.2 & LReLU  \\
Conv2D  & 1    & (1,1) & (1,1) & N  & 0.2 & Sigmoid\\
\hline   
\end{tabular}
\caption{\textbf{Architecture for CelebA Discriminator.} $D_x$, $D_z$ and $D_{xz}$ are the discriminator components to process the image $x$, the latent variable $z$ and the output of $D_x$ and $D_z$ concatenated, respectively.}
\label{CDisc}
\end{table}


\section{Morpho-MNIST Latent Space Interpolations}
We also plot latent space interpolations between pairs of images sampled from the test set. Figure \ref{mnistint} shows that the model has learned meaningful latent space representations and the transitional images look realistic as well. 

\begin{figure}[!htb]
\centering
\includegraphics[width = 0.9\columnwidth]{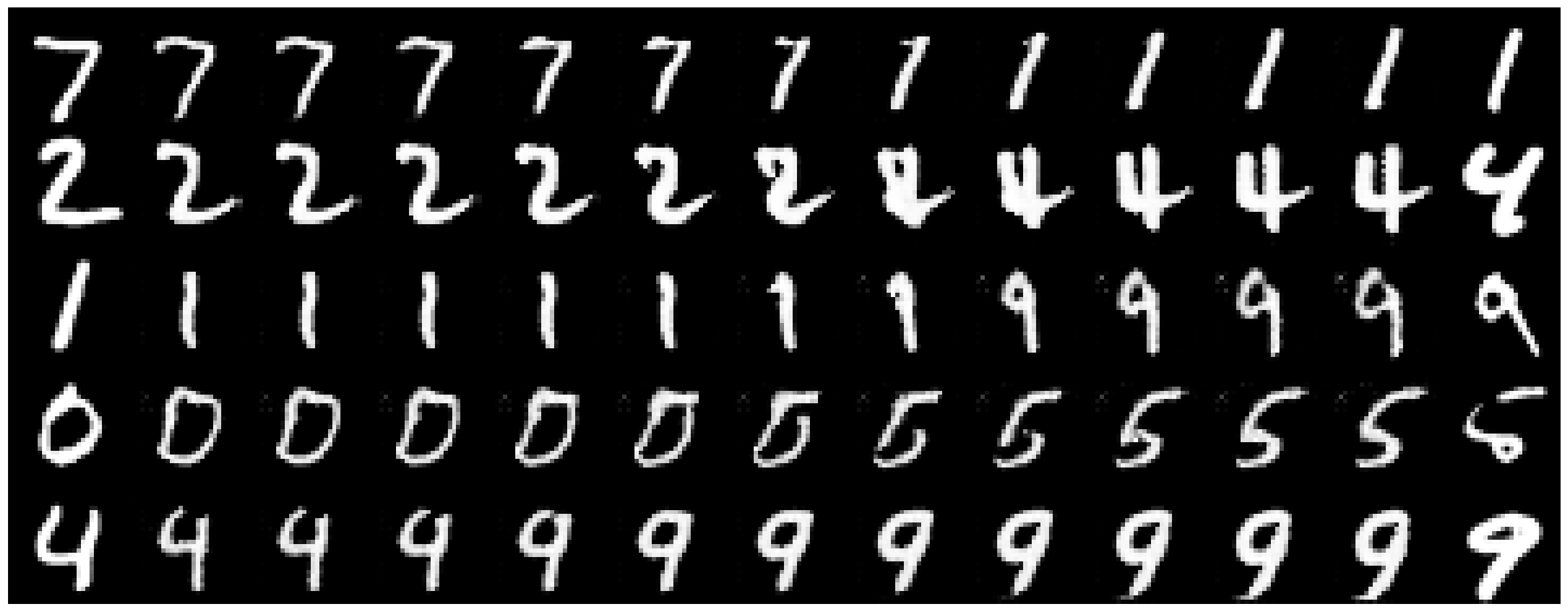}
\caption{\textbf{Morpho-MNIST Interpolations.} The columns on the extreme left and right denote real samples from the test set and the columns in between denote generated images for the linearly interpolated latent space representation $\textbf{z}$.}
\label{mnistint}
\end{figure}

\section{Morpho-MNIST Reconstructions}
We qualitatively evaluate the inference model (Encoder) by sampling images along with their attributes from the test set and passing them through the encoder to obtain their latent space representations. These representations are passed to the generator which outputs reconstructions of the original image. The reconstructions are showed in Figure \ref{mnistrec}. Overall, reconstructions for Morpho-MNIST are faithful reproductions of the real images.
\begin{figure}[!htb]
\centering
\includegraphics[width = 0.8\columnwidth]{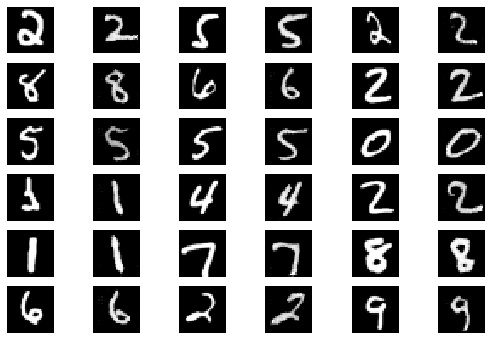}
\caption{\textbf{Morpho-MNIST Reconstructions.} Odd columns denote real images sampled from the test set, and even columns denote reconstructions for those real images.}
\label{mnistrec}
\end{figure}

\section{Morpho-MNIST: Evaluating Label CFs}
As shown in Figure~\ref{mnistcnt}, we empirically evaluated the counterfactual generation on Morpho-MNIST by generating CFs that change the digit label for an image. To check whether the generated counterfactual image corresponds to the desired digit label, we use the output of a digit classifier trained on Morpho-MNIST images. Here we provide details of this pre-trained digit classifier.
The classifier architecture is shown in Table \ref{mnistclf}. The classifier converges in approximately 1k iterations with a validation accuracy of 98.30\%.
\begin{table}[!htb]
\centering

\begin{tabular}{lllclll}
\hline 
Layer     & F                     & K                     & \multicolumn{1}{l}{S}     & BN & D   & A   \\
\hline \\
Conv2D    & 32                    & (3,3)                 & \multicolumn{1}{l}{(1,1)} & N  & 0.0 & ReLU                  \\
Conv2D    & 64                    & (3,3)                 & \multicolumn{1}{l}{(1,1)} & N  & 0.0 & ReLU                  \\
MaxPool2D & \multicolumn{1}{c}{-} & (2,2)                 & -                         & N  & 0.0 & \multicolumn{1}{c}{-} \\
Dense     & 256                   & \multicolumn{1}{c}{-} & -                         & N  & 0.5 & ReLU                  \\
Dense     & 128                   & \multicolumn{1}{c}{-} & -                         & N  & 0.5 & ReLU                  \\
Dense     & 10                    & \multicolumn{1}{c}{-} & -                         & N  & 0.5 & Softmax       \\
\hline
\end{tabular}
\caption{\textbf{Morpho-MNIST Label Classifier}}
\label{mnistclf}
\end{table}
We then use the classifier to predict the labels of the counterfactual images and compare them to the labels that they were supposed to be changed to (target label).  Overall, 97.30\% of the predicted labels match the target label. Since the classifier is not perfect, the difference between the CF image's (predicted) label and the target label may be due to an error in the classifier's prediction or due to an error in CF generation, but this is minimal. We show some images for which predicted labels do not match the target labels in Figure \ref{mnistwrong}. Most of these images are digits that can be confused as another digit; for instance, the first image in row 2 is a counterfactual image with a target label of $9$, but was predicted by the digit classifier as a $1$.
\begin{figure}[!htb]
\centering
  \includegraphics[width = 0.8\columnwidth]{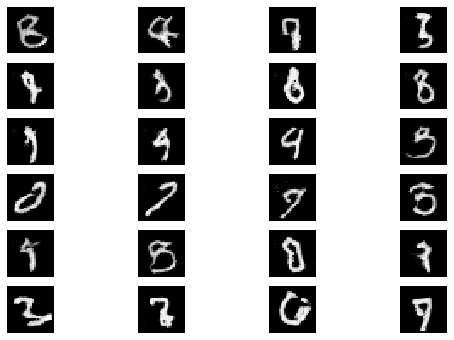}
  \caption{\textbf{Misclassified Counterfactuals.} Counterfactual images for Morpho-MNIST on which target label does not match predicted label.}
  \label{mnistwrong}
\end{figure}

\section{DeepSCM Implementation on CelebA}

Since DeepSCM was not implemented on CelebA, we reproduce their method based on their available code and choose an appropriate VAE architecture for CelebA data. Tables \ref{CDec_DSCM} and \ref{CEnc_DSCM} show the Decoder and Encoder architectures respectively for generating CelebA counterfactuals. Specifically, we use Equation A.4 from the DeepSCM paper \cite{pawlowski2020deep} to model the image $\vecx$. We use a conditional VAE with a fixed variance Gaussian decoder to output a \textit{bias} to further reparametrize a Gaussian distribution using a location-scale transform (as stated in \cite{pawlowski2020deep}).
In our implementation of DeepSCM, we do not need to model the attributes $\veca$ using Normalizing Flows, since the attributes independently influence the images in the underlying SCM in case of CelebA (as explained in Section 5). 
We set the latent dimension to 256, use the Adam optimizer \cite{kingma2014adam} with a learning rate of 0.0005 to train the conditional VAE with a batch size of 128. Similar to \citet{pawlowski2020deep}, we adopt a constant variance assumption and set $\log \sigma^2 = -5$. We also preprocess the images by scaling them between [0,1); the preprocessing flow in Equation A.4 in \cite{pawlowski2020deep} is hence altered by normalizing its outputs between [0,1).

\begin{table}[!htb]
\centering

\begin{tabular}{lllllll}
\hline
Layer   & F   & K     & S   & BN & D   & A \\
\hline \\
Dense & 4096 & - & - & N  & 0.0 & Linear  \\
Conv2DT & 256 & (3,3) & (2,2) & Y  & 0.25 & LReLU  \\
Conv2DT & 128 & (3,3) & (2,2) & Y  & 0.25 & LReLU  \\
Conv2DT & 64 & (3,3) & (2,2) & Y  & 0.25 & LReLU  \\
Conv2DT & 32  & (3,3) & (2,2) & Y  & 0.25 & LReLU  \\
Conv2D & 3   & (3,3) & (1,1) & Y  & 0.0 & Sigmoid\\
\hline 
\end{tabular}
\caption{\textbf{Architecture for CelebA DeepSCM Decoder}}
\label{CDec_DSCM}
\end{table}

\begin{table}[!htb]
\centering

\begin{tabular}{lllllll}
\hline
Layer  & F   & K     & S     & BN & D   & A \\
\hline
Conv2D & 32  & (3,3) & (2,2) & Y  & 0.25 & LReLU  \\
Conv2D & 64 & (3,3) & (2,2) & Y  & 0.25 & LReLU  \\
Conv2D & 128 & (3,3) & (2,2) & Y  & 0.25 & LReLU  \\
Conv2D & 256 & (3,3) & (2,2) & Y  & 0.25 & LReLU  \\
Conv2D & 256 & (1,1) & (1,1) & N  & 0.0 & LReLU  \\
\hline   
\end{tabular}
\caption{\textbf{Architecture for CelebA DeepSCM Encoder}}
\label{CEnc_DSCM}
\end{table}

\section{Additional DeepSCM vs \ourmethod\ CelebA Counterfactuals}
In addition to Figure \ref{icfgen_dscm} of the main paper, we provide additional counterfactual images obtained using \ourmethod\ and DeepSCM in Figure \ref{icfgen_dscm_more}. As observed in Section 5.2, the counterfactuals generated using \ourmethod\ are more true to the corresponding intervention than DeepSCM.
\begin{figure}[!htb] 
\centering
\includegraphics[width = \columnwidth]{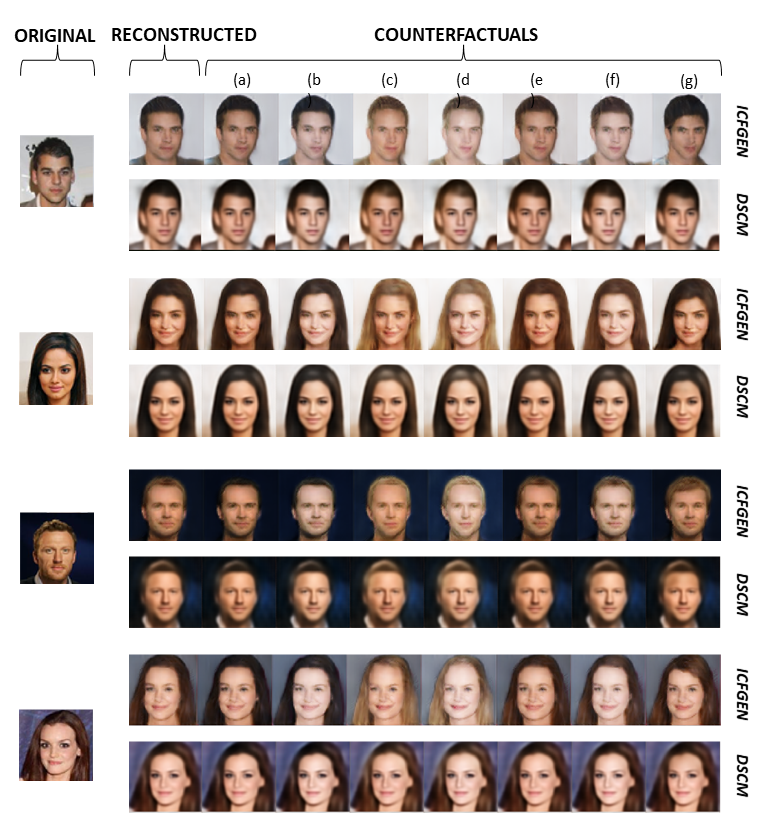}
\caption{\textbf{\ourmethod\ and DeepSCM Counterfactuals.} (a) denotes $do$ (black hair = 1) and (b) denotes $do$ (black hair = 1, pale =1). Similarly (c) denotes $do$ (blond hair = 1); (d) denotes $do$ (blond hair = 1, pale = 1); (e) denotes $do$ (brown hair = 1); (hf denotes $do$ (brown hair = 1, pale = 1); and (g) denotes $do$ (bangs = 1).}
\label{icfgen_dscm_more}
\end{figure}

\section{Additional \ourmethod\ CelebA Counterfactuals}
Figure \ref{celebacnt} shows more CF examples obtained using \ourmethod.
Note how the CF image is different w.r.t. the base (reconstructed) image only in terms of the intervened attribute. Consequently, if the attribute in the base image is already present, the CF image is exactly the same as the original image. For instance, (Ib) and (Ig) in Fig \ref{celebacnt} are exactly the same since (Ia) already has brown hair, hence intervening on brown hair has no effect on the hair color. 
\begin{figure}[!htb]
\centering
  \includegraphics[width=1.0\columnwidth]{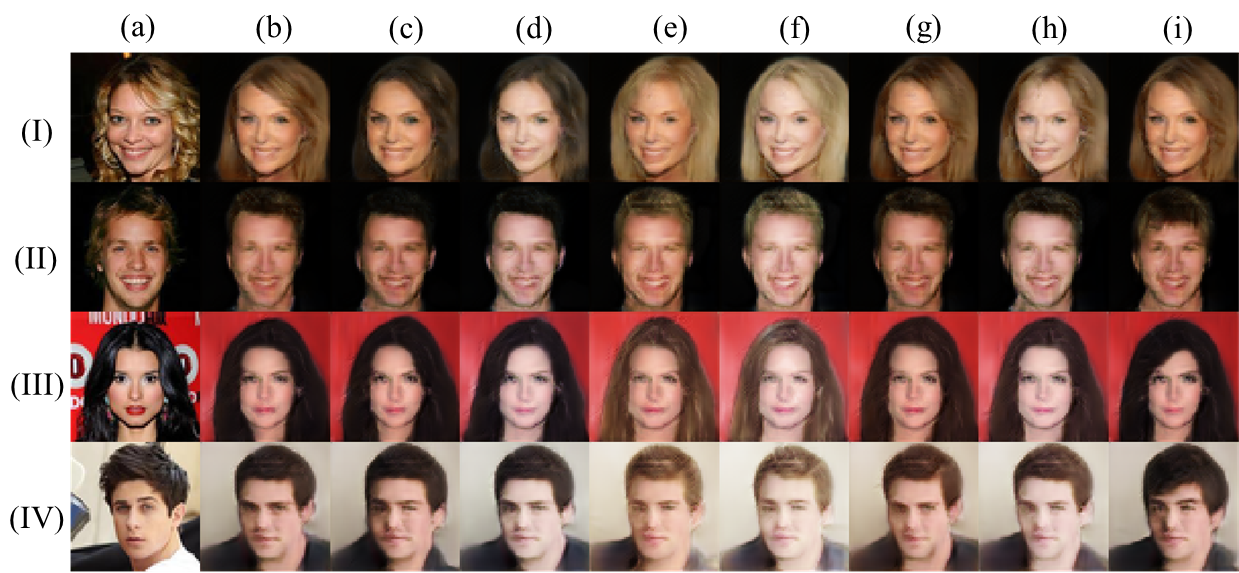}
  \caption{\textbf{CelebA Counterfactuals.} Column (a) across all rows I, II, III, IV represent the real image and (b) represents the reconstructed image. (c) denotes $do$ (black hair = 1) and (d) denotes $do$ (black hair = 1, pale =1). Similarly (e) denotes $do$ (blond hair = 1); (f) denotes $do$ (blond hair = 1, pale = 1); (g) denotes $do$ (brown hair = 1); (h) denotes $do$ (brown hair = 1, pale = 1); and (i) denotes $do$ (bangs = 1).}
  \label{celebacnt}
\end{figure}

\section{Ablation Study of Style-Based Generator and Cyclic Cost Minimization}
We plot the reconstructed images of real images, produced by ALI, StyleALI (ALI with style based generator) and Cyclic Style ALI (Style ALI with Cyclic Cost Minimization) in Figure \ref{celebaablation}. We observe that using a style-based generator significantly improves the quality of the generated images and applying the cyclic cost minimization algorithm on top of it improves reconstruction of the real image.
\begin{figure}[!htb]
\centering
  \includegraphics[width=1.0\columnwidth]{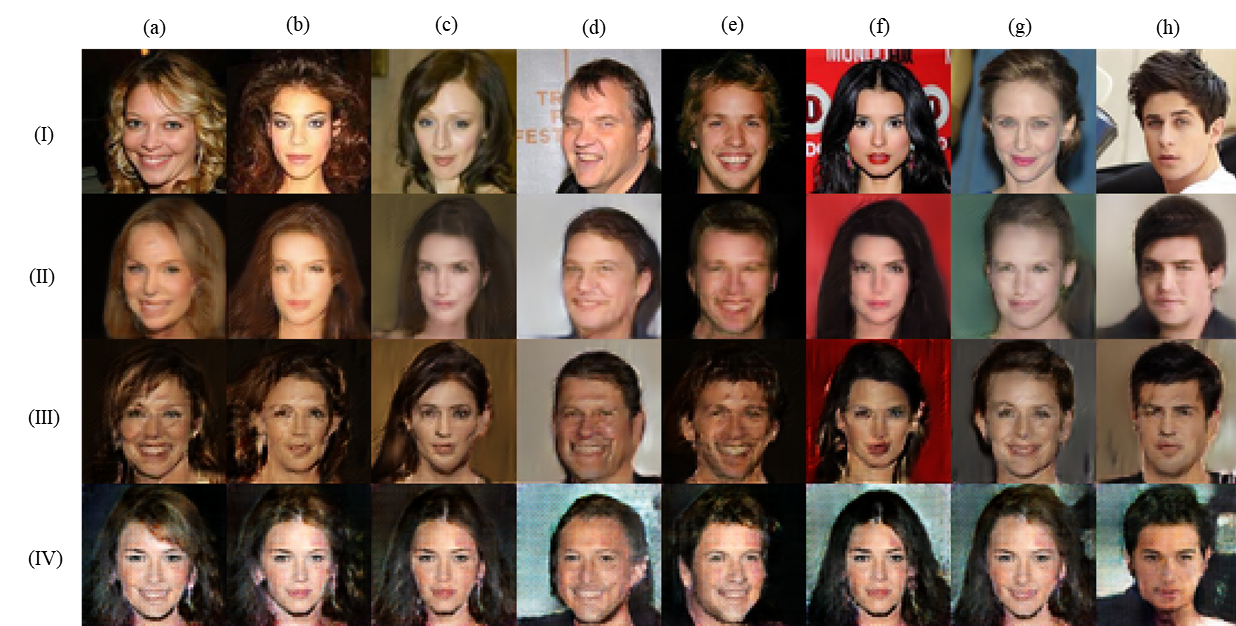}
  \caption{\textbf{Ablation Study.}
  Row (I) consists of real images, (II) consists of reconstructions generated by Cyclic Style ALI (CSALI), (III) Style ALI reconstructions and (IV) has ALI reconstructions }
  \label{celebaablation}
\end{figure}

\section{Human Evaluation of Counterfactuals}
To quantitatively evaluate the counterfactuals, we asked human evaluators to pick the ``edited" (counterfactual) image from 10 randomly sampled pairs of reconstructed and counterfactual images to human evaluators and asked them to pick the \textit{edited} (counterfactual) image. Overall, we got 66 responses resulting in an average score of 5.15 correct answers out of 10 with a standard deviation of 1.64, indicating that the generated counterfactual distribution is perceptually indistinguishable from the reconstructed  one. Figure \ref{human_eval} shows sample questions on the form circulated for our human evaluation studies. For each question, the user chose one of two images that seems edited to the human eye. If the counterfactual can fool human perception, it indicates better performance of counterfactual generation.`
\begin{figure}[!htb]
    \centering
    \includegraphics[width=0.9\columnwidth]{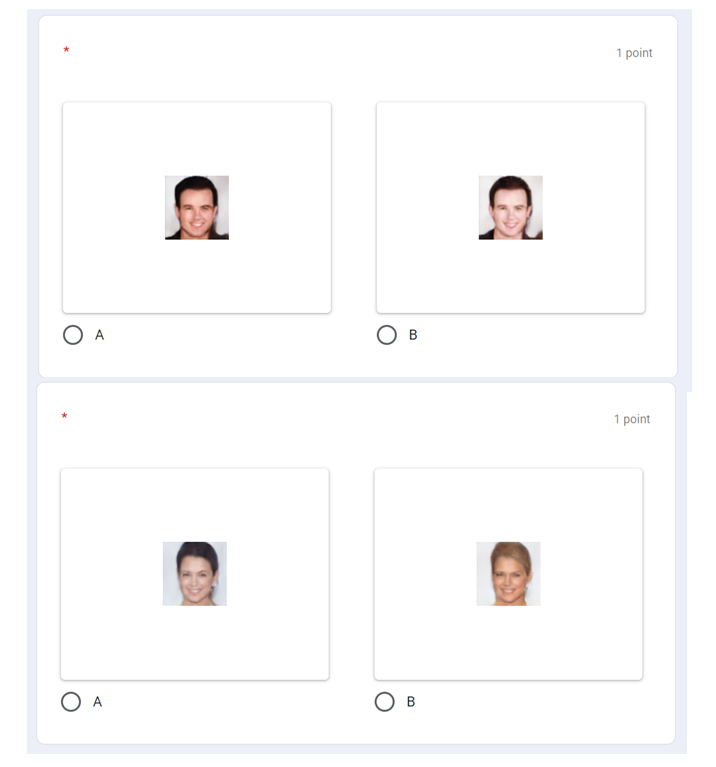}
    \caption{\textbf{Edited or Not?} Sample questions on the form circulated for human evaluation studies}
    \label{human_eval}
\end{figure}

\section{{\em Attractive} Classifier Architecture}
We describe the architecture and training details for the {\em attractiveness} classifier whose fairness was evaluated in  Figure~\ref{celebafair}. The same classifier's output was explained in Figure~\ref{attrrank}.   \textit{Attractive} is a binary attribute associated with every image in CelebA dataset. The architecture for the binary classifier is  shown in Table \ref{attractiveclf}. 
\begin{table}[!htb]

\centering
\begin{tabular}{lllllll}
\hline
Layer  & F   & K     & S     & BN & D   & A \\
\hline
Conv2D & 64  & (2,2) & (1,1) & Y  & 0.2 & LReLU  \\
Conv2D & 128 & (7,7) & (2,2) & Y  & 0.2 & LReLU  \\
Conv2D & 256 & (5,5) & (2,2) & Y  & 0.2 & LReLU  \\
Conv2D & 256 & (7,7) & (2,2) & Y  & 0.2 & LReLU  \\
Conv2D & 512 & (4,4) & (1,1) & Y  & 0.5 & LReLU  \\
Conv2D & 512 & (1,1) & (1,1) & Y  & 0.5 & Linear \\
Dense & 128 & - & - & N  & 0.5 & ReLU \\
Dense & 128 & - & - & N  & 0.5 & ReLU \\
Dense & 1 & - & - & N  & 0.5 & Sigmoid \\
\hline
\end{tabular}
\caption{\textbf{Architecture for the \textit{Attractiveness} Classifier}}
\label{attractiveclf}
\end{table}
We use the Adam optimizer with a learning rate of $10^{-4}$, $\beta_1 = 0.5$ and a batch size of 128 for training the model. For the LeakyReLU activations, $\alpha = 0.02$. The model converges in approximately 20k iterations. All weights are initialized using the Keras truncated normal initializer with $mean = 0.0$ and $stddev = 0.01$. All biases are initialized with zeros.

\section{Bias Mitigation Details}
We now present more details of our bias mitigation results in continuation to the discussion in Sec 5.4 of the main paper. To pick the optimal model while finetuning the classifier with the proposed bias mitigation regularizer, we set an accuracy threshold of 80\% and pick the model with the lowest bias. We use $\lambda=1.0$ and use an Adam optimizer with a learning rate of $10^{-4}$, $\beta_1 = 0.5$ and a batch size of 128 for training the model. 
\begin{table}[!htb]
\centering

\begin{tabular}{l|l|l}
        & Fair Classifier  & Biased Classifier \\
        \hline
        \rule{0pt}{2.6ex}black\_h, pale &  0.032 & 0.159 \\
blond\_h, pale & -0.041 & 0.077 \\
brown\_h, pale & 0.012 & 0.154
\end{tabular}
\caption{\textbf{Bias Values after Bias Mitigation.} Lower bias is better. Absolute bias values less than 5\% are not considered significant.}
\label{table3}
\end{table}

\section{Complex Attribute SCM}
Our method can also be utilized in the setting where the attributes are connected in a complex causal graph structure, unlike \cite{denton2019detecting, joo2020gender}. We now conduct a fairness analysis w.r.t age for the \textit{Attractive} classifier, assuming that \textit{Young} affects other visible attributes like \textit{Gray hair}.

\begin{figure}[!htb]
\centering
  \includegraphics[width = 0.6\columnwidth]{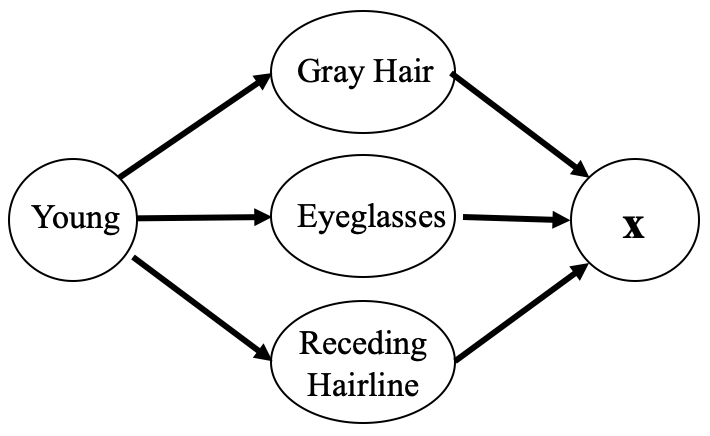}
  \caption{\textbf{\textit{Young} Attribute-SCM.}
  SCM that describes the relationship between the attribute \textit{Young} and the generated image.}
  \label{celebayoung}
\end{figure}

\noindent \textit{Defining the Causal Graph.}
Fig \ref{celebayoung} shows a potential causal graph for the attribute \textit{Young}. While we ignored the attribute \textit{Young} in the previous experiments, we now propose a SCM that describes the relationship between \textit{Young} and a given image, mediated by a few facial attributes. Among the two demographic attributes, we choose \textit{Young} over \textit{Male} since \textit{Young} has some descendants that can be agreed upon in general, whereas using \textit{Male} would have forced us to explicitly model varying societal conventions surrounding this attribute. For e.g., the SCM proposed by \cite{kocaoglu2018causalgan} describes a causal relationship from \textit{gender} to \textit{smiling} and \textit{narrow eyes}, which is a problematic construct. For the attribute \textit{Young}, we specifically choose \textit{Gray Hair}, \textit{Eyeglasses} and \textit{Receding Hairline} since it is reasonable to assume that older people are more likely to have these attributes, as compared to younger ones. 


\noindent \textit{Learning the Attribute SCM.}
To learn the model parameters for the SCM shown in Figure~\ref{celebayoung}, we estimate conditional probabilities for each edge between \textit{Young} and the other attributes using maximum likelihood over the training data, as done in Bayesian networks containing only binary-valued nodes~\cite{pearl2011bayesian}. This SCM in Figure~\ref{celebayoung}, combined with the SCM from Fig \ref{celebagraph} that connects other facial attributes to the image, provides us the \textit{augmented} Attribute SCM that we use for the downstream task of CF generation. Once the Attribute SCM is learned, the rest of the steps in Algorithm 1 remain the same as before.  That is, based on this augmented Attribute SCM, the counterfactuals are generated according to the Eqn \ref{counterfactual}. For example,  to generate a CF changing \textit{Young} from 1 to 0, the Prediction step involves changing the values of gray hair, receding hairline and eyeglasses based on the modified value of \textit{Young}, according to the learned parameters (conditional probabilities) of the SCM.

\begin{figure}[!htb]
\centering
  \includegraphics[width = 0.9\columnwidth]{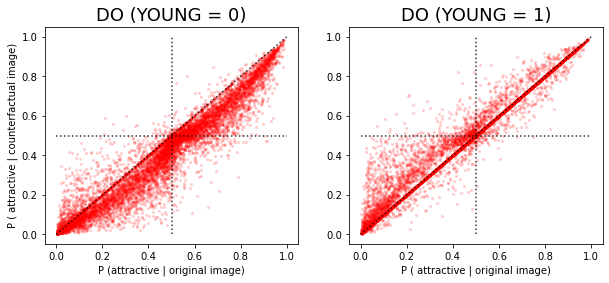}
  \caption{\textbf{Fairness Analysis for Complex SCM.}  Counterfactual images that change \textit{Young} from 1 to 0 (left), have a lower prediction score from the \textit{attractiveness} classsifier, while counterfactual images that change \textit{Young} from 0 to 1 (right), have a slightly higher prediction 
  score.}
  \label{celebacomplex}
\end{figure}

\noindent \textit{Fairness Analysis.}
We conduct a fairness analysis similar as above, using the above attractiveness classifier. We generate counterfactuals for \textit{Young = 1} and \textit{Young = 0} according to the causal graph in Fig \ref{celebayoung}. The analysis is showed in Fig \ref{celebacomplex}; we observe that the classifier is evidently biased against \textit{Young = 0} and biased towards \textit{Young = 1}, when predicting \textit{attractive=1}. We quantify this bias using Eqn \ref{bias-simple}. Using counterfactuals that change \textit{Young} from 1 to 0, we get a negative bias of -0.136 and for changing \textit{Young} from 0 to 1, we get a positive bias of 0.087 which are both substantial biases assuming a threshold of 5\%. Therefore, given the causal graph from Figure~\ref{celebayoung}, our method is able to generate counterfactuals for complex high-level features such as age, and use them to detect any potential biases in a machine learning classifier.

\end{document}